\documentclass[aip,reprint,footinbib,reqno]{amsart}

\usepackage{graphicx,graphics,caption, subcaption, enumerate,color}

\usepackage{amsmath,amssymb,amsfonts,bm,empheq}
\usepackage{cases}

\usepackage{xcolor} 
\usepackage{soul} 

\usepackage[ruled,vlined]{algorithm2e}

\usepackage{epstopdf}

\usepackage[a4paper]{geometry} 
\usepackage{hyperref}
 \hypersetup{
    colorlinks=true,       
    linkcolor=red,          
     citecolor=green,        
     urlcolor=cyan           
 }

\usepackage{setspace}

\usepackage{diagbox} 
\usepackage{multirow}

\usepackage{soul}
\soulregister\ref7
\soulregister\cite7
\soulregister\eqref7
\usepackage{changes}
\usepackage{todonotes}

\usepackage{lineno}

\newlength\mylen

\newcommand{\Real}{\mathbb {R}}

\newcommand{\p}{\mathbb{P}}
\newcommand{\e}{ \operatorname{E}}

 \newcommand{\q}{ \mathbb{Q}}

\newtheorem{definition}{Definition}[section]
\newtheorem{theorem}{Theorem}

\newtheorem{remark}{Remark}
\newtheorem{property}[theorem]{Property}

\def\kpa2_2{\frac{\kappa^2}{2}}

\def\sigm2_2{\frac{\sigma^2}{2}}
\def\nab_dot{\nabla\cdot}

\def\1/N{\frac{1}{N}}

\def\H_1_u{H_u^{-1}}

\def\f12{\frac{1}{2}}

\def\e{\mathbf{E}}

\begin{document}
 
\begin{center}
{\bf \Large Learn Quasi-stationary Distributions of Finite State Markov Chain}

\

Zhiqiang Cai\footnote{Corresponding author.  School of Data Science, City University of Hong Kong, Tat Chee Ave, Kowloon, Hong Kong SAR. Email: zqcai3-c@my.cityu.edu.hk }
, Ling Lin\footnote{School of Mathematics, Sun Yat-sen University, Guangzhou 510275, China. Email: linling27@mail.sysu.edu.cn}
, Xiang Zhou\footnote{ School of Data Science and Department of Mathematics, City University of Hong Kong, Tat Chee Ave, Kowloon, Hong Kong SAR. Email: xiang.zhou@cityu.edu.hk}

\end{center}
 \date{\today}

\section*{Abstract}
We propose a reinforcement learning (RL) approach to compute the expression of quasi-stationary distribution.
Based on the fixed-point formulation of quasi-stationary distribution,
we  minimize the KL-divergence of two Markovian path distributions induced
by the candidate distribution and the true target distribution.
To solve this challenging minimization problem by gradient descent, we apply the reinforcement learning technique by introducing
the reward and value functions.
We  derive the corresponding policy gradient theorem and design an  actor-critic algorithm to learn
the optimal solution and the value function. The numerical examples of finite state Markov chain are tested to demonstrate the  new method.

\

{{\bf Keywords}: quasi-stationary distribution, reinforcement learning, KL-divergence, actor-critic algorithm }
\


\section{Introduction}

Quasi-stationary distribution (QSD) is  the long time statistical behavior of a stochastic process that will be surely killed when this process is conditioned to survive \cite{collet2012quasi}. This concept has been widely used applications, such as in biology and ecology \cite{buckley2010analytical,lambert2008population}, chemical kinetics \cite{de2004quasi,dykman1995statistical}, epidemics \cite{artalejo2013stochastic,clancy2011approximating,sani2007stochastic}, medicine \cite{chan2009quantitative} and neuroscience \cite{berglund2012mixed,landon2012perturbation}.  Many works for the rare events  in meta-stable systems also focus  on this quasi-stationary distribution \cite{di2017jump,lelievre2015low}. In addition, some new Monte Carlo sampling methods, for instance, the Quasi-stationary Monte Carlo method  \cite{pollock2016scalable,wang2020approximation} also arise by using the QSD
instead of true stationary distribution. 

We are interested in the numerical computation of QSD and focus on the finite state Markov chain in this paper.
Mathematically, the quasi-stationary distribution can be solved as the principal left eigenvector of a sub-Markovian transition matrix. So, the traditional numerical algebra methods can be applied to solve the quasi-stationary distribution in finite state space, for example, the power method \cite{watkins2004fundamentals}, the multi-grid  method \cite{bebbington1997parallel} and  Arnoldi's algorithm \cite{pollett1994efficient}. These eigenvector methods can produce the 
stochastic vector for QSD, instead of generating samples of QSD.

In search of   efficient algorithms for large state space, stochastic   approaches are  in favor of either sampling the QSD or computing the expression of QSD, and these methods can be   applied or extended easily to continuous state space too.
A  popular  approach for sampling   quasi-stationary distribution is the Fleming-Viot stochastic method \cite{martinez1994quasi}. The Flemming-Viot method first simulates $N$ particles independently. When any one of the particles falls into the absorbing state and gets killed, a new particle is uniformly selected from the remaining $N-1$ surviving particles to replace the dead one, and the simulation continues.
When the time and $N$ tend to infinity,
the particles' empirical distribution can converge to the quasi-stationary distribution. 

In \cite{aldous1988two,benaim2015stochastic,de2005simulate}, the authors proposed   to recursively update the expression of QSD at each iteration based on the empirical distribution of a single-particle  simulation. It is shown \cite{benaim2015stochastic} that the convergence rate can be $O(n^{-1/2})$, where 
 $n$ is the iteration number. This method is later  improved in \cite{blanchet2016analysis,zheng2014stochastic} by applying the stochastic approximation method \cite{kushner2003stochastic} and the Polyak-Ruppert averaging technique \cite{polyak1992acceleration}. These improved algorithms have a choice of flexible step size  but needs a projection operator onto probability simplex, which carries some extra computational overhead increasing with the number  of states.    \cite{wang2020approximation}  extends the algorithm to the diffusion process.


In this paper, we focus on how to compute the
expression of the quasi-stationary distribution, which is denoted by $\alpha(x)$ on a metric space $\mathcal{E}$.
If $\mathcal{E}$ is finite, $\alpha$ is a probability vector and if $\mathcal{E}$ is 
a domain in $\Real^d$, then $\alpha$ is a probability density function on $\mathcal{E}$.
We assume $\alpha$ can be numerically represented in parametric form $\alpha_\theta$, $\theta\in\Theta$.
 This family  $\{\alpha_\theta\}$ can be  in  tabular form or any neural network.
 Then the problem of finding the QSD $\alpha$ becomes how to  compute the optimal parameter $\theta$ in $\Theta$. We call this problem the learning problem for the QSD. 
In addition, we want to directly learn the QSD,
not to use the distribution family $\{\alpha_\theta\}$ to fit the simulated samples generated by other traditional simulation methods. 
 
Our minimization problem for   QSD is similar to the variational inference (VI)\cite{Blei2017VI}, which minimizes an objective functional measuring the distance between the target and candidate distributions. However, unlike the mainstream VI methods such as  evidence lower bound (ELBO) technique\cite{Jordan1999VI} or particle-based\cite{NIPS2016_b3ba8f1b}, flow-based methods \cite{pmlr-v37-rezende15},
our approach is based on recent important progresses  from  reinforcement learning(RL) method\cite{sutton2018reinforcement}, particularly the policy gradient method and actor-critic algorithm. 
We first regard the learning process of the   quasi-stationary distribution as the interaction with the environment
which is constructed by the property of QSD. 
Reinforcement learning has recently shown tremendous advancements 
and remarkable successes in applications  (e.g.\cite{mnih2013playing,silver2018general,popova2018deep}). The RL framework provides an innovative and powerful  modelling and computation approach for many scientific computing problems.

The essential question is how   to  formulate the QSD problem as an RL problem. 
Firstly,  for the sub-Markovian kernel $K$ 
of a Markov process,
we can define a Markovian kernel
$K_\alpha$ on $\mathcal{E}$ (see Definition \ref{def:K}) and then
the QSD is defined by the equation $\alpha=\alpha K_\alpha$, which 
equals $\alpha$ as the initial distribution      
and the distribution after one step.
Secondly,  
we consider an optimal $\alpha$ (in our parametric family of distribution) to  minimize the Kullback–Leibler divergence (i.e., relative entropy) of two path distributions, denoted by $\mathbb{P}$ and $\mathbb{Q}$, associated with two 
Markovian kernels $K_\alpha$ and $K_\beta$
where $\beta:=\alpha K_\alpha$.
 Thirdly,
 inspired by the  recent work \cite{rose2020reinforcement} 
 of using RL for rare events sampling problems, we transform the minimization of KL divergence between $\mathbb{P}$ and $\mathbb{Q}$ into the maximization of a time-averaged reward function and define the corresponding value function $V(x)$ at each  state $x$. 
 This completes our modelling of 
 RL for the quasi-stationary distribution problem. 
 Lastly, we  derive the policy gradient theorem (Theorem \ref{th:grad}) to compute the gradient w.r.t.  $\theta$ of the averaged reward for the   learning dynamic for the averaged reward. This is known as the ``actor" part.  The  ``critic" part is to learn the value function $V$ in its parametric form $V_\psi$.
 The actor-critic algorithm uses the stochastic gradient descent to train the parameter $\theta$ for the action $\alpha_\theta$ and the parameter $\psi$ for the value function $V_\psi$. See Algorithm \ref{dac1}.
  
Our contribution is that we first devise a method to transform the QSD problem into the RL problem. Similar to \cite{rose2020reinforcement}, our paper also uses the KL-divergence to define the RL problem. However, our paper fully adapts the unique property of QSD that is a fixed point problem $\alpha=\alpha K_\alpha$ to define the RL problem.
   
  Our learning method allows the
  flexible parametrization of the distributions and uses
  the stochastic gradient method to   train the optimal distribution. It is easy to implement  optimization  with scale up to large state spaces.
  The  numerical examples we tested have shown our methods converge faster than the other existing methods \cite{de2005simulate,blanchet2016analysis}. 
  
  Finally, we remark that 
  our method works very well for QSD of the strict sub-Markovian kernel $K$,  but not  applicable to compute the invariant distribution when $K$ is Markovian. This is 
  because we transform the problem into the variational problem between two Markovian kernels $K_\alpha$ and $K_\beta$ (where $\beta =\alpha K_\alpha$). Note $K_\alpha(x,y)=K(x,y)+(1-K(x,\mathcal{E}))\alpha(y)$ (Definition \ref{def:K}) and our method is based on the fact that $\alpha=\beta$   if and only if $K_\alpha=K_\beta$. If $K$ is Markovian kernel, then $K_\alpha\equiv K$ for any $\alpha$, and our method can not work. So  $K(x,\mathcal{E})$ has to be strictly less than 1 for some $x \in \mathcal{E}$.


This paper is organized as follows. Section \ref{s2} is a short review of the quasi-stationary distribution and some basic simulation methods of QSD. In Section \ref{s3}, we first formulate the reinforcement learning problem by KL-divergence and derive the policy gradient theorem (Theorem \ref{th:grad}). Using the above formulation, we then develop the actor-critic algorithm to estimate the quasi-stationary distribution.  In Section \ref{s5}, the efficiency of our algorithms is illustrated by four examples compared with the simulation methods in \cite{zheng2014stochastic}.

\section{Problem Setup and 
Review }\label{s2}
\subsection{Quasi-stationary Distribution}
We start with an abstract setting.
Let $\mathcal{E}$ be a finite state equipped with the Borel $\sigma$-field $\mathcal{B}(\mathcal{E})$, and  let $\mathcal{P}(\mathcal{E})$  be the space of probabilities over $\mathcal{E}$. A sub-Markovian kernel on $\mathcal{E}$ is defined as 
a map $K:\mathcal{E}\times\mathcal{B}(\mathcal{E}) \mapsto [0,1]$ such that for all $x\in\mathcal{E}, A\mapsto K(x,A)$ is a nonzero measure with $K(x,\mathcal{E})\leq 1$ and for all $A\in\mathcal{B}(\mathcal{E}),x\mapsto K(x,A)$ is measurable. Particularly, if $K(x,\mathcal{E})=1$ for all $x\in\mathcal{E}$, then $K$ is called a Markovian kernel. Throughout the paper, assume that $K$ is strictly sub-Markovian, i.e., $K(x,\mathcal{E})<1$ for some $x$.

Let  $X_t$   be a Markov chain with values in $\mathcal{E}\cup \left\{\partial\right\}$ where $\partial\notin \mathcal{E}$ denotes an absorbing state. We define the extinction time   
$$
\tau:= \inf\left\{t>0 : \quad X_t=\partial\right\}.
$$
We define the \textbf{quasi-stationary distribution(QSD)} $\alpha$ as the long time limit of the conditional distribution, if there exists a probability distribution $\nu$ on $\mathcal{E}$ such that:
\begin{equation}
\alpha(A):=\lim _{t \rightarrow \infty} P_{\nu}\left(X_{t} \in A \mid \tau>t\right),
\qquad  A\in \mathcal{B}(\mathcal{E}).
\end{equation}
where $P_\nu$ refers to the probability
distribution of $X_t$ associated with the initial distribution $\nu$ on $\mathcal{E}$.
Such a conditional distribution well describes the behavior of the process before extinction and it is 
easy to see that $\alpha$ satisfies the following fixed point problem
\begin{equation} \label{eqn:QSD}
    {P}_{\alpha}\left(X_{t} \in A \mid \tau>t\right)=\alpha(A)
\end{equation}
where $P_\alpha$ refers to the probability distribution of $X_t$ associated with the initial distribution $\alpha$ on $\mathcal{E}$. \eqref{eqn:QSD} is equivalent to the following stationary condition that 
\begin{equation}
    \alpha=\frac{\alpha K}{\alpha K \mathbf{1}}, \quad \mbox{or  }
\alpha(y)=\frac{\sum_{x} \alpha(x) K(x,y)}{ \sum_{x}\alpha(x)K(x,\mathcal{E})}
\end{equation}
where 
$\alpha$ is a row vector and $\mathbf{1}$ denotes the column vector
with all entries being one
and $$K(x,\mathcal{E})=\sum_{x'\in \mathcal{E}}K(x,x').$$

For any  sub-Markovian  kernel  $K$,
we can associate with $K$  a Markovian kernel $\tilde{K}$ on $\mathcal{E}\cup\{\partial\}$ defined by \begin{equation*}
    \left\{
    \begin{array}{l}
         \tilde{K}(x,A)=K(x,A)  \\
         \tilde{K}(x,\{\partial\})=1-K(x,\mathcal{E})\\
         \tilde{K}(\partial,\{\partial\})=1.
    \end{array}
    \right.
\end{equation*}
for all $x \in \mathcal{E}, A \in \mathcal{B}(\mathcal{E})$.
The kernel $\tilde{K}$ can be understood as the Markovian transition kernel of the Markov chain $(X_t)$ on $\mathcal{E}\cup\{\partial\}$ whose transitions in $\mathcal{E}$ is specified by $K$, but  is ``killed" forever once it leaves $\mathcal{E}$.

In this paper, we assume 
$\mathcal{E}$ is a finite state space and 
 the process in consideration 
has a unique QSD. Assume that $K$ is irreducible, then existence and uniqueness of the quasi-stationary distribution can be obtained by the Perron-Frobenius theorem \cite{meleard2012quasi}.

An important Markovian kernel
is the following $K_\alpha$
which is defined on $\mathcal{E}$ only and has
a ``regenerative probability'' $\alpha$.

\begin{definition}\label{def:K}
For any given $\alpha \in \mathcal{P}(\mathcal{E})$ and
 a sub-Markovian kernel $K$ on
 $\mathcal{E}$, we define $K_{\alpha}$,  a  Markovian kernel on $\mathcal{E}$, as follows
\begin{equation}
\label{eqn:K_alpha}
    K_{\alpha}(x,A):= K(x,A)+\left( 1-K(x,\mathcal{E})\right)\alpha(A)
\end{equation}
for all $x\in\mathcal{E}$ and $A\in\mathcal{B}(\mathcal{E})$.
\end{definition}
$K_\alpha$  is a Markovian kernel because   
$K_\alpha(x,\mathcal{E})=1$.
It  is easy to sample  
$X_{t+1}\sim K_\alpha(X_t,\cdot)$ from 
any state  $X_t\in  \mathcal{E}$: 
Run  the transition as normal by using $\tilde{K}$ to have a next state denoted by $Y$, then $X_{t+1}=Y$ if $Y\in \mathcal{E}$, otherwise, sample 
$X_{t+1}$ from  $\alpha$.

We know that $\alpha$ is the quasi-stationary distribution of $K$ if and only if it is the stationary distribution of $K_{\alpha}$, i.e., 
\begin{equation} \label{fp:Ka}
   \alpha = \alpha K_{\alpha}.
\end{equation}
It is easy to see
$\alpha=\beta$ if and only if
$K_\alpha=K_\beta$ for any two distributions $\alpha$ and $\beta$. Also, for every $\alpha'$, $K_{\alpha'}$ has a unique invariant probability denoted by $\Gamma(\alpha')$. Then $\alpha'\mapsto\Gamma(\alpha')$ is continuous in $\mathcal{P}(\mathcal{E})$ (i.e. for the topology of weak convergence) and    there exists $\alpha\in\mathcal{P}(\mathcal{E})$ such that $\alpha=\Gamma(\alpha)$ or, equivalently,   $\alpha$ is a QSD for $K$. 
  
\subsection{Review of 
simulation methods for quasi-stationary distribution}\label{ssec:sim}
According to the above subsection, 
the QSD
$\alpha$ satisfies the fixed point problem 
\begin{equation}
\label{eqn:fixed-point}
    \alpha=\Gamma(\alpha),
\end{equation}
where $\Gamma(\alpha)$
is the stationary distribution of $K_\alpha$ on $\mathcal{E}$.
In general, \eqref{eqn:fixed-point} can be solved 
recursively by $
    \alpha_{n+1}\leftarrow\Gamma(\alpha_n)
$.

The Fleming–Viot (FV) method 
\cite{martinez1994quasi} evolves $N$ particles   independently of each other as a Markov process associated with the transition kernel $K_\alpha$ until one  
succeeds in jumping to the absorbing state $\partial$. At that time, this killed particle is immediately  reset to $\mathcal{E}$ as an 
initial state uniformly chosen from one of the remaining $N-1$ particles. 
The QSD 
$\alpha$ is approximated by the empirical distribution of the $N$ particles in total
and these particles can be regarded as samples from the quasi-stationary distribution $\alpha$ like the MCMC method.

\cite{blanchet2013empirical}
proposed a simulation method by
only  using one particle  at each iteration
to update   $\alpha$. At iteration $n$, given an $\alpha_n\in\mathcal{P}(\mathcal{E})$,  one  can run the discrete-time Markov chain $X^{(n+1)}$ as  normal on $\partial\cup \mathcal{E}$ with
initial $X^{(n+1)}_0\sim \alpha_n$,
then $\alpha_{n+1}$ is computed as the following  weighted average of empirical distributions: 
\begin{equation}\label{eq:iter}
 \begin{split} 
\alpha_{n+1}(x):
=\alpha_{n}(x)+\frac{1}{n+1} \sum_{k=0}^{\tau^{(n+1)}-1} \frac{I\left(X_{k}^{(n+1)}=x \mid X_{0}^{(n+1)} \sim \alpha_{n}\right)-\alpha_{n}(x)}{\frac{1}{n+1} \sum_{j=1}^{n+1} \tau^{(j)}}
\end{split}
\end{equation}
where $n\geq0$ and $I$ is the indicator function,
$\tau^{(j)}=\min \left\{k \geq 0 \mid X_{k}^{(j)} \in \partial \right\}$ is the first extinction time for the process $X^{(j)}$.
This iterative scheme 
has the convergence rate $O(\frac{1}{\sqrt{n}})$.

In \cite{blanchet2016analysis,zheng2014stochastic}, the above method is extended to the stochastic approximations framework 
\begin{equation}\label{eqn:proj}
\alpha_{n+1}(x)=\Theta_{H}\left[\alpha_{n}+\epsilon_{n}\sum_{k=0}^{\tau^{(n+1)}-1}
\left(I\left(X_{k}^{(n+1)}=x\vert X_{0}^{(n+1)} \sim \alpha_{n}\right)-\alpha_{n}(x)\right)\right]
\end{equation}
where $\Theta_{H}$ denotes the $L_2$ projection  into the probability simplex and $\epsilon_{n}$ is the step size satisfying  $\sum{\epsilon_{n}}=\infty$ and $\sum{\epsilon^2_{n}}<\infty$. 
Specifically, if $\epsilon_{n} = O(\frac{1}{n^{r}})$ for $0.5<r<1$, under sufficient condition, they have
$
\sqrt{n^{r}}\left(\alpha_n-{\alpha}\right)\xrightarrow{\text{d}}\mathcal{N}(0,V)
$ for some   matrix $V$\cite{blanchet2016analysis,zheng2014stochastic}.
If the Polyak-Ruppert averaging technique is applied to generate 
\begin{equation}
\label{eqn:PR}
\nu_{n}:=\frac{1}{n} \sum_{k=1}^{n} \alpha_{k},
\end{equation}
then the convergence rate of $\nu_n\to \alpha$ becomes $\frac{1}{\sqrt{n}}$ \cite{blanchet2016analysis,zheng2014stochastic}.

The simulation schemes 
 \eqref{eq:iter} and \eqref{eqn:proj} 
 need to sample the initial states according to $\alpha_n$ and 
 to add
 the empirical distribution 
 and $\alpha_n$ at each $x$ pointwisely.
 So they are suitable for finite state space where $\alpha$ is a probability vector saved in the tabular form. 
In \eqref{eqn:proj}
there is no need to record all exit times  $\tau^{(j)},j=1,\dots,n$,
but the additional projection operation in \eqref{eqn:proj}  is computationally expensive since the cost is  $O(m\log m)$ where $m=|\mathcal{E}|$  \cite{boyd2004convex,wang2013projection}.

\section{Learn Quasi-stationary Distribution}\label{s3}

 We focus on the computation of the expression of the quasi-stationary distribution. Particularly, when this distribution is parametrized in a certain way by  $\theta$, we can extend the tabular form for finite-state Markov chain to any flexible form, even in the neural networks for probability density function in $\Real^d$.  But we do not pursue this representation and expressivity issue here, and   restrict our discussion to finite state space  only  to illustrate our main idea first. 
 In finite state space, $\alpha(x)$ for $x\in \mathcal{E}=\{1,\dots,m\}$, can 
 be simply described as a softmax function with $m-1$ parameter $\theta_i:\alpha(i)\propto e^{\theta_i},  1\leq i\leq m-1$, ($\theta_m=0$). This introduces no representation error. For the  generalization to continuous space $\mathcal{E}$ in jump and diffusion processes, or even for a  huge finite state space, a good representation of  $\alpha_\theta(x)$ is important in practice. 
 
 In this section,
  we shall formulate our QSD problem in terms of reinforcement learning (RL) so that the 
  problem of seeking optimal parameters becomes a policy optimization problem.
  We derive the policy gradient theorem to construct a gradient descent method for the optimal parameter. We then show how to design the actor-critic algorithms based on stochastic optimization.

\subsection{Formulation of RL and Policy Gradient Theorem}\label{s3_1}
Before introducing the RL method of our QSD problem, we develop a general formulation in terms of RL by introducing  the KL-divergence between two path distributions.

 Let $P_\theta$ and $Q_\theta$ be two
 families of Markovian  kernels on $\mathcal{E}$ in parametric forms
with the same set of parameters $\theta\in\Theta$.
Assume both $P_\theta$ and $Q_\theta$ are ergodic for any $\theta$.
Let $T>0$
and denote a path up to time $T$ by  $\omega_0^T=(X_0,X_1,\ldots, X_T)\in \mathcal{E}^{T+1}$. Define the   path distributions
under the Markov chain kernel $P_\theta$ and $Q_\theta$, respectively:
\begin{equation}\label{eqn:ptheta}
    \p_{\theta}(\omega_0^T):=\prod_{t=1}^{T}P_{\theta}(X_t\mid X_{t-1}), \quad 
      \q_{\theta}(\omega_0^T):=\prod_{t=1}^{T}Q_{\theta}(X_t\mid X_{t-1}) .
\end{equation}
 Define  the KL divergence  from 
 $\p_\theta$ to $\q_\theta$ on $\mathcal{E}^{T+1}$
\begin{equation}\label{eqn:DKL}
     \mathsf{D}_{KL}(\p_{\theta}\mid \q_\theta):=\sum_{\omega_{0}^T}{\p_{\theta}(\omega_{0}^T)\ln{\frac{\p_{\theta}(\omega_0^T)}{\q_\theta(\omega_0^T)}}}
    =-\e_{P_\theta} \sum_{t=1}^T R_\theta (X_{t-1},X_{t}),
\end{equation}
where  the expectation $\e_{P_\theta}$
is for the path 
$(X_0,X_1,\ldots,X_T)$
generated by the transition
kernel $P_\theta$ and
\begin{equation}
    \label{def:R}
    R_\theta(X_{t-1},X_t):= -\ln{\frac{P_{\theta}(X_t\mid X_{t-1})}{Q_{\theta}(X_{t}\mid X_{t-1})}}.
\end{equation}
is called the  (one-step) {\bf reward}.

Define the  {\bf{ average reward }}
$r(\theta)$
as the time averaged negative KL divergence in the limit of $T\to\infty$:
\begin{align}\label{eqn:dkl}
   r(\theta)
    &:=-\lim_{T\to \infty}{\frac{1}{T} \mathsf{D}_{KL}
    (\p_{\theta}\mid \q_{\theta})}
    = -\lim_{T\to \infty} \frac{1}{T}  \e_{P_\theta} \sum_{t=1}^T R_\theta (X_{t-1},X_{t}).
\end{align}
Due to ergodicity of $P_\theta$, $r(\theta)=\sum_{x_0,x_1}   R_\theta(x_0,x_1)P_\theta(x_1\vert x_0)   \mu_\theta(x_0) $
where $\mu_\theta$ is the invariant measure 
of  $P_\theta$.    
 $r(\theta)$ is independent of initial state $X_0$.
Obviously $r(\theta)\leq 0$ for any $\theta$.
\begin{property} The following are equivalent:  
\begin{enumerate}
    \item  $r(\theta)$ reaches its maximal value $0$ at $\theta^*$.
    \item 
    $\p_{\theta^*}=\q_{\theta^*}$ in $\mathcal{P}(\mathcal{E}^{T+1})$ for any $T>0$.
\item 
$P_{\theta^*}=Q_{\theta^*}$.
\item $R_{\theta^*}\equiv  0$.
\end{enumerate}  
 \end{property}
 
\begin{proof}
We only need to show $(1)\Longrightarrow (3)$.
It is easy to see
$$r(\theta)=-\sum_{x_0} ~\mathsf{D}_{KL}(P_\theta(\cdot \vert x_0)
~ \vert ~Q_\theta(\cdot \vert x_0)) ~ \mu_\theta(x_0).$$
If $r(\theta)=0$, 
since $\mu_\theta > 0$,
then
$$\mathsf{D}_{KL}(P_\theta(\cdot \vert x_0) 
 ~\vert~ Q_\theta(\cdot \vert x_0))=0~~\quad \forall x_0.$$
So we have $P_\theta=Q_\theta$.

\end{proof}

The above property establishes the relationship between the RL problem and QSD problem.
  
We show our theoretic main result below as the foundation of our algorithm to be developed later.
This theorem can be regarded as  one type of  the policy gradient theorem for policy gradient method in reinforcement learning  \cite{sutton2018reinforcement}.

 Define   the {\bf  value function} (\cite{sutton2018reinforcement} Chapter 13):
 \begin{equation} \label{def:V}
V \left(x\right):= \lim _{T \rightarrow \infty}\sum_{t=1}^{T} \e_{ P_{\theta}}\left[R_\theta(X_{t-1},X_t)-r(\theta)    \mid X_{0}=x\right].
\end{equation}
Certainly, $V$ also depends on $\theta$, though we do not write $\theta$ explicitly.

\begin{theorem}[policy gradient theorem]\label{th:grad}
 We have the following two 
 properties
 \begin{enumerate}
     \item  At any $\theta$, for any $x\in \mathcal{E}$, the following Bellman-type equation  holds  for the value function $V$ and the average reward $r(\theta)$:
  \begin{equation}
  \label{eqn:V}
     V (x) 
     = \e_{Y\sim P_\theta(\cdot\mid x)}\left[V (Y) +R_\theta(x,Y) -r(\theta)\right] .
  \end{equation}
     \item 
     The gradient of the average reward ${r}(\theta)$ is 
\begin{align}\label{eqn:gddkl}
    \nabla_{\theta} r(\theta)&=\e\left[ \nabla_{\theta} \ln Q_{\theta}(Y\mid X)\right]+\notag\\
      & \qquad\e  \left[
      \bigg(V (Y)-V (X)+R_\theta(X, Y)-r(\theta)\bigg)\nabla_{\theta}\ln P_{\theta} (Y \mid X ) \right],
\end{align}
 \end{enumerate}
where   the expectations    are for the joint distribution $(X,Y)\sim \mu_\theta(x)P_{\theta}(y\mid x)$ where $\mu_\theta$ is the stationary measure of $P_\theta$.
\end{theorem}
\begin{proof}
We shall prove the Bellman equation first and then we   use the Bellman equation to derive the gradient of the average reward ${r}(\theta)$.
For any $x_0\in \mathcal{E}$, by writing $\omega_0^T=(x_0,\ldots,x_T)$ and defining
$$\Delta R_\theta(\omega_0^T)=\sum_{t=1}^T (R(x_{t-1}, x_t)-r(\theta))$$
we have
\begin{align}\label{eqn:bellman}
    V\left(x_0\right)
    &=\lim _{T \rightarrow \infty} \e_{P_\theta}\left[\Delta R_\theta(\omega_0^T) \mid X_{0}=x\right]
    \notag
    \\
    &=\lim_{T\to\infty}\sum_{x_2,\ldots, x_T}\sum_{x_1}
    \left( 
    \left(\prod_{t=2}^{T}P_{\theta}(x_t\mid x_{t-1}) \right)P_{\theta}(x_1\mid x_0)\Delta R(\omega_0^T) \right)\notag
    \\
    &=\lim_{T\to\infty}\sum_{x_1}
    \left( P_{\theta}(x_1\mid x_0)\sum_{x_2,\ldots, x_T}
    \left(
     \prod_{t=2}^{T}P_{\theta}(x_t\mid x_{t-1})\left[\Delta R(\omega_1^T)+\Delta R(\omega_0^1)\right]
    \right) \right)\notag
    \\
    &=\sum_{x_1}
    \left( P_{\theta}(x_1\mid x_0) \left(
    \lim_{T\to\infty}
    \left[\sum_{x_2,\ldots,x_T}\prod_{t=2}^{T}P_{\theta}(x_t\mid x_{t-1})\Delta R(\omega_1^T)\right]+\Delta R(\omega_0^1)\right)\right) \notag\\
    &=\sum_{x_1}P_{\theta}(x_1\mid  x_0)\left[ V(x_1)+R_\theta(x_0, x_1)\right] -r(\theta),
\end{align}
which proves \eqref{eqn:V}, i.e.,
$$
r(\theta)= \e_{Y\sim P_\theta(\cdot\mid x)}\left[V (Y) + R_\theta(x,Y)- V (x)\right],
\qquad \forall x\in \mathcal{E}.
$$
 
 Next, we compute the gradient of $r(\theta)$. 
By  the trivial  equality    
  \begin{equation} \label{eqn:ir}
      \sum_{x_1}P_\theta(x_1\mid x_0) \nabla_\theta \ln P_\theta (x_1\mid  x_0)=\nabla_\theta 
  \sum_{x_1}P_\theta(x_1\mid x_0)=0, 
  \end{equation}  
  and the definition \eqref{def:R}, 
 we can write the gradient of  $r(\theta)$ as follows 
 \begin{align*}
    \nabla_{\theta} r(\theta)=& \sum_{y} \nabla_{\theta} P_{\theta}(y \mid x)\left[V(y)+R_\theta(x, y)-V(x)\right] \\
    &
    +\sum_{y} P_{\theta}\left(y \mid x\right)\left[\nabla_{\theta} V(y)-\nabla_{\theta} V(x)+\nabla_{\theta}\ln{Q_\theta(y\mid x)}\right].
\end{align*}
We here keep  the term $V(x)$  in the first line, even though it has no contribution here
(in fact, to add any constant to  $V(x)$ is also fine). 
Since this equation holds for all states $x$ on the right-hand side, we take the expectation w.r.t. $\mu_\theta$,  the stationary distribution  of $P_\theta$. 
So,     we have 
\begin{align*}
    \nabla_{\theta} r(\theta)
    =&\sum_{x,y} \mu_{\theta}(x)\nabla_{\theta} P_{\theta}\left(y \mid x\right)\left[V(y)+R_\theta(x, y)-V(x)\right] \\
     &+\sum_{x,y} \mu_{\theta}(x)P_{\theta}\left(y \mid x\right)\left[\nabla_{\theta} V(y)-\nabla_{\theta} V(x)+\nabla_{\theta}\ln{Q_\theta(y\mid x)}\right]\\
     =&\sum_{x,y} \mu_{\theta}(x)\nabla_{\theta} P_{\theta}\left(y \mid x\right)\left[V(y)+R_\theta(x, y)-V(x)\right] \\
     &+\sum_{y}\mu_{\theta}(y)\nabla_{\theta}V(y)-\sum_{x}\mu_{\theta}(x)\nabla_{\theta}V(x)+
     \sum_{x,y} \mu_{\theta}(x)P_{\theta}\left(y \mid x\right)\nabla_{\theta}\ln{Q_\theta(y\mid x)}\\
     =&\sum_{x,y} \mu_{\theta}(x)P_{\theta}\left(y \mid x\right)\bigg[V(y)+R_\theta(x, y)-V(x)\bigg]~\nabla_{\theta}\ln{P_{\theta}\left(y \mid x\right)}\\
     &+\sum_{x,y} \mu_{\theta}(x)P_{\theta}\left(y \mid x\right)\nabla_{\theta}\ln{Q_\theta(y\mid x)}.
\end{align*}
In fact,  we can add any constant number $b$ (independent of  $x$ and $y$) inside the 
squared bracket of the last line  without changing the equality, due to 
  the following fact   similar to   \eqref{eqn:ir}:  $
    \sum_{x,y} \mu_{\theta}(x)\nabla_{\theta} P_{\theta}\left(y \mid x\right) =\sum_{y}\mu_{\theta}(y) \nabla_{\theta} \sum_{x}P_{\theta}\left(x \mid y\right)=0
$.  \eqref{eqn:gddkl}
is a special case of $b=r(\theta)$.
\end{proof}

\begin{remark}
As shown in the proof,
 \eqref{eqn:gddkl} holds
 if   $r(\theta)$ 
 at the right-hand side 
 is replaced by any constant number $b$.   $b=r(\theta)$ is a good choice to reduce the variance since  $r(\theta)$ can be regarded as the expectation of $R_\theta$.
\end{remark}

\begin{remark}
If $P_\theta=Q_\theta$, then
the first term of \eqref{eqn:gddkl} vanishes due to \eqref{eqn:ir}
and the second term of \eqref{eqn:gddkl} vanishes 
due to \eqref{eqn:V}. 

\end{remark}
\begin{remark}
 The name of ``policy'' here refers to the role of  $\theta$ as the policy for decision makers to improve the reward $r(\theta)$.
\end{remark}

\subsection{Learn QSD}
Now we discuss how to 
connect the QSD with the results in the previous subsection.
In view of  equation   \eqref{fp:Ka},
we introduce $\beta:=\alpha K_\alpha$ as 
the one-step distribution if starting from the initial $\alpha$,  i.e.,
\begin{equation}
\label{eqn:beta}
\beta (y):=\sum_{x\in\mathcal{E}}{\alpha(x)K_{\alpha}(x,y)}, ~~\quad \forall y
\end{equation}

By \eqref{fp:Ka}, 
  $\alpha$ is a QSD if and only if  $\beta=\alpha$.
  However, we do not directly
  compare  these two distributions
  $\alpha$ and $\beta$.
  Instead, we consider their Markovian kernels
  induced by \eqref{eqn:K_alpha}: $K_\alpha$ and $K_\beta$.
Our approach is to consider the KL divergence similar to \eqref{eqn:DKL} between two kernels
$K_\alpha$ and $K_\beta$ since 
$\alpha=\beta$ if and only if
$K_\alpha = K_\beta$. 
In this way, one can view $K_\alpha$ and $K_\beta$ (note $\beta=\alpha K_\alpha$) as 
two transition matrices 
$P_\theta$ and $Q_\theta$ in the previous section, in which the parameter $\theta$ here is in fact the distribution $\alpha$.

To have a further representation of the distribution $\alpha$, which is a (probability mass) function on  $\mathcal{E}$,
we   
 propose a parametrized family    for  $\alpha$ in the form 
$\alpha_\theta$ where $\theta$ is a generic parameter. In the simplest case, $\alpha_\theta$ takes the so-called {\it soft-max}
form $\alpha_\theta(i)=\frac{e^{\theta_i}}{\sum_{j\geq 1} e^{\theta_j}}$
if $\mathcal{E}=\{1,\ldots,N\}$
for $\theta=(\theta_1,\ldots,\theta_{N-1},\theta_{N}\equiv 0).$
This parametrization represents $\alpha$ without any approximation error for finite state space and 
the effective space of $\theta$ is just $\Real^{N-1}$.
For certain problems, particularly  with large state space, if one has some prior knowledge about the structure of the function $\alpha$ on $\mathcal{E}$, one might propose other parametric forms of $\alpha_\theta$ with the dimension of $\theta$ less than the cardinality  $|\mathcal{E}|$ 
to improve the efficiency, although the extra representation error in this way has to be introduced.

For any given $\alpha_\theta\in \mathcal{P}(\mathcal{E})$,    the corresponding Markovian kernel $K_{\alpha_\theta}$ is then defined in  \eqref{eqn:K_alpha}
and  $\beta_\theta=\alpha_\theta K_{\alpha_\theta}i$ is defined by \eqref{eqn:beta}.
$K_{\beta_\theta}$ is like-wisely defined  by \eqref{eqn:K_alpha} again.
 To use the formulation in Section \ref{s3_1}, we choose   $P_\theta = K_{\alpha_\theta}$ and $Q_\theta = K_{\beta_\theta}$.  
Define the objective function as before: 
\begin{align*}
   {r}(\theta)
    &:=-\lim_{T\to \infty}{\frac{1}{T} \mathsf{D}_{KL}
    (\p_{\theta}\mid \q_{\theta})}
    = -\lim_{T\to \infty} \frac{1}{T}  \e_{P_\theta} \sum_{t=1}^T {R}_\theta(X_{t-1},X_{t}).
\end{align*}
where $${R}_\theta(x,y)=-\ln{\frac{K_{\alpha_\theta}(x,y)}{K_{\beta_\theta}(x,y)}}.$$


The   value function  $V(x)$ is defined like-wisely.
Theorem \ref{th:grad} now gives the expression of gradient
\begin{equation}\label{eqn:gdac1}
\begin{split}
        \nabla_\theta r(\theta) & =   \e[\big(R_\theta(X,Y)-r(\theta)+V(Y)-V(X)\big) \nabla_\theta \ln K_{\alpha_\theta}(X,Y)
        \\ &\qquad +\nabla_\theta  \ln{K_{\beta_\theta}(X,Y)}]
    \end{split}
\end{equation}
where     $(X,Y)\sim \mu_\theta(x)K_{\alpha_
\theta}(x,y)$ where $\mu_\theta$ is the stationary measure of $K_{\alpha_\theta}$.

The optimal $\theta^*$
for the QSD $\alpha_\theta$
is to  maximize $r(\theta)$ 
and this can be solved by   the gradient descent algorithm
\begin{equation}
\label{eqn:theta_t}
\theta_{t+1} = \theta_t+\eta_t^\theta \nabla_\theta r(\theta_t).
\end{equation}
where $\eta_t^\theta>0$ is the step size.
In practice, the stochastic gradient is applied  
$$\nabla_\theta r(\theta_t) \approx
{\nabla_\theta \ln K_{\alpha_\theta}(X_t,X_{t+1}) \times 
\delta (X_t,X_{t+1})+ \nabla_\theta \ln{K_{\beta_\theta}(X_t,X_{t+1})}}
$$
where $X_t, X_{t+1}$ are sampled
based on the Markovian kernel $K_{\alpha_\theta}$ ( see Algorithm \eqref{dac1})
and   the differential temporal (TD) error $\delta_t$ is 
\begin{equation}
\label{TD}
\delta_t=\delta(X_t,X_{t+1}) = R_\theta(X_t,X_{t+1}) - r(\theta_t) + V (X_{t+1}) - V (X_t).
\end{equation}

Next, we need to address a remaining issue to address:
how to compute the value function $V$ and $r(\theta_t)$
in the TD error \eqref{TD}.
Besides, we also need to show the details of computing $\nabla_\theta K_{\alpha_\theta}$ and $\nabla_\theta K_{\beta_\theta}$.

\subsection{Actor-Critic Algorithm}

 With the stochastic gradient method \eqref{eqn:theta_t}, we can obtain the optimal policy $\theta^*$. We refer to \eqref{eqn:theta_t} as the learning dynamics for the policy and it is generally known as actor.
 To calculate the value function $V$  appearing in $\nabla r(\theta)$, we need to have a new learning dynamics, which is called {\it critic}. Then the overall policy-gradient method is termed as the actor-critic method.

  We start  with
    the Bellman equation \eqref{eqn:V} for   the value function and consider   the mean-square-error loss 
\begin{equation*}
    \operatorname{MSE}[V]=\frac12 \sum_{x}\nu(x)\left(\sum_{y}K_{\alpha_\theta}(x,y)\left[V(y)+R_\theta(x,y)-r(\theta)\right]-V(x)\right)^{2} 
\end{equation*}
where  $\nu$ is any distribution
supported on $\mathcal{E}$.  
$ \operatorname{MSE}[V]=0$ if and only if $V$ satisfies the Bellman equation \eqref{eqn:V}, i.e. $V$ is the value function. To learn $V$, 
   we   introduce  the function approximation for the value function, $V_{\psi}$, with the parameter $\psi$ and 
   consider to   minimize 
\begin{equation*}
    \operatorname{MSE}(\psi)=\frac12 \sum_{x}\nu(x)\left(\sum_{y}K_{\alpha_\theta}(x,y)\left[V(y)+R_\theta(x,y)-r(\theta)\right]-V_{\psi}(x)\right)^{2} 
\end{equation*}
by  the semi-gradient method (\cite{sutton2018reinforcement}, Chapter 9):
\begin{align*}
     \nabla_{\psi}\operatorname{MSE}(\psi)&=-\sum_{x,y}\nu(x)K_{\alpha_\theta}(x,y)\left[V(y)+R_\theta(x,y)-r(\theta)-V_{\psi}(x)\right]\nabla_{\psi}V_{\psi}(x)\\
     &\approx-\sum_{x,y}\nu(x)K_{\alpha_\theta}(x,y)\left[V_{\psi}(y)+R_\theta(x,y)-r(\theta)-V_{\psi}(x)\right]\nabla_{\psi}V_{\psi}(x)
\end{align*}
Here the term $V(y)$ is
frozen first and then approximated by $V_\psi$ since it could be treated as a  prior guess of the value  function for the future state.

Then for the gradient descent
iteration $
\psi_{t+1} =\psi_t - \eta_t^\psi  \nabla_{\psi}\operatorname{MSE}_{V}(\psi_t)
$
 where $\eta_t^\psi$ is
the step size,
we can have the stochastic gradient iteration
\begin{equation} \label{psit}
   \psi_{t+1}=\psi_t+\eta_{t}^{\psi}~ \delta(X_t,X_{t+1})~ \nabla_{\psi}V_{\psi_t}(X_t)
\end{equation}
   where the differential temporal (TD) error $\delta$ defined above in \eqref{TD}:
 $$  \delta_t=\delta(X_t,X_{t+1}) = R_{\theta_t}(X_t,X_{t+1}) - r(\theta_t) + V_{\psi_t} (X_{t+1}) - V_{\psi_t} (X_t) .$$
Here for simplicity,  $(X_t,X_{t+1})$ are the same samples as in the actor method for $\theta_t$.  This means that the distribution $\nu$ above is chosen as $\mu$ used for the gradient $\nabla_\theta r(\theta)$.

   Next, we consider  
 the calculation of  the   reward $r(\theta)$. Again by   the Bellman equation \eqref{eqn:V}
$$
\sum_{x}\mu(x)\sum_{y} K_{\alpha_\theta} (x,y) (R_\theta(x,y)-r(\theta)+V (y)-V (x)) = 0 
$$
Let $r_t$ be the estimate of the   reward $r(\theta_t)$ at time $t$. We can update our estimate of the reward every time a transition occurs as
\begin{equation}
\label{rt}
r_{t+1} = r_t +\eta_t^r \times \delta_t
\end{equation}
 where $\delta_t$ is the TD error before
$$
\delta_t=\delta(X_t,X_{t+1}) = R_{\theta_t}(X_t,X_{t+1}) - r_t + V_{\psi_t} (X_{t+1}) - V_{\psi_t} (X_t).
$$

In conclusion,  \eqref{eqn:theta_t}\eqref{psit}\eqref{rt} together
consist of the actor-critic algorithm,
which is summarized in Algorithm \ref{dac1}. We remark that    Algorithm \ref{dac1} 
  can be easily adapted to use the mini-batch gradient method where several 
copies of $(X_t,X_{t+1})$ are sampled 
and the average is used to update the parameters.
The stationary distribution
$\mu_{\theta}$ of $K_{\alpha_\theta}$ is 
sampled by running the corresponding Markov chain  
for several steps with ``warm start'':  the initial for $\theta_{t+1}$ is set   as the final state generated from the previous iteration at  $\theta_t$. The length of
this ``burn-in'' period 
 can be set as just one step
 in practice for efficiency.

\begin{algorithm}[th]
  \SetKwData{Left}{left}\SetKwData{This}{this}\SetKwData{Up}{up}
  \SetKwFunction{Union}{Union}\SetKwFunction{FindCompress}{FindCompress}

  \textbf{Initialization}\\
  $t=0$;
  $\theta = \theta_0$;
  $\psi = \psi_0$; 
  $r_t=r_0$;
  \\
  Sample  $X_0\sim \mu_{\theta_0},$     the stationary distribution of
    $K_{\alpha_{\theta_0}}$
  \\
  \For{$t=0,1,2,\dots $}
   {
   Sample $X_{t+1}$ from the transition kernel $K_{\alpha_{\theta_t}}(X_t,X_{t+1})$\\
   $
   \delta_t=V_{\psi_t}(X_{t+1})-V_{\psi_t}(X_t)+R_{\theta_t}(X_t,X_{t+1})-r_t
   $\\
   $
   \theta_{t+1}=\theta_{t}+\eta_{t}^{\theta}
   ~( \delta_t    \nabla_\theta \ln K_{\alpha_{\theta_t}}(X_t,X_{t+1})     +  \nabla_\theta  \ln  K_{\beta_{\theta_t}}(X_t,X_{t+1}))
   $
   \\
   $
   \psi_{t+1}=\psi_t+\eta_{t}^{\psi}~ \delta_t \nabla_{\psi}V_{\psi_t}(X_t)
   $\\
   $
   r_{t+1}=r_t+\eta_{t}^{r}~ \delta_t
   $
   \\
   $
   X_{t} \sim \mu_{\theta_{t+1}}
   $ the stationary distribution of $K_{\alpha_{\theta_{t+1}}}$
   \\
   $   t=t+1   $
   }
  \caption{
  ({\bf  ac-$\alpha$ method})  Actor Critic algorithm for quasi-stationary distribution $\alpha_\theta$}\label{dac1}
\end{algorithm}
\begin{remark}
Finally, we remark the computation of $\nabla_\theta\ln{K_{\alpha_\theta}}$ and $\nabla_\theta\ln{K_{\beta_\theta}}$ in Algorithm \ref{dac1}. The details are shown in Appendix. We comment that the main computational cost is the function $K(x,\mathcal{E})$, which has to be pre-computed and stored. If the problem has some special structure, the function could be approximated in parametric form. Another special case is our example 2 where $K(x,\mathcal{E})=0\quad\forall x\in\{2,3,\dots,N\}$.
\end{remark}

\section{Numerical experiment}\label{s5}
In this section, we present two examples to demonstrate \textbf{Algorithm} \ref{dac1}.
We call    the algorithm \eqref{eq:iter}, \eqref{eqn:proj} and \eqref{eqn:PR} in Section \ref{ssec:sim}  used in  \cite{blanchet2016analysis,zheng2014stochastic}, 
as   \textbf{Vanilla Algorithm}, \textbf{Projection Algorithm} and \textbf{Polyak Averaging Algorithm} respectively.
Let    0 be the absorbing state and $\mathcal{E}=\{1,\ldots,N\}$ are non-absorbing
states,   the Markov transition matrix on $\{0,\ldots, N\}$ is denoted  by
$$
\tilde{K}=
\left[\begin{array}{cc}
1 & 0  \\
* & K
\end{array}\right],
$$ 
where $K$ is an $N$-by-$N$ sub-Markovian  matrix.
For   \textbf{Algorithm}  \ref{dac1}, the  distribution $\alpha_\theta$ on $\mathcal{E}$ is always parameterized as 
$$
\alpha_\theta=\frac{1}{e^{\theta_1}+\ldots+e^{\theta_{N-1}}+1}\left[e^{\theta_1},\ldots,e^{\theta_{N-1}}, 1\right],
$$
and the value function  
$V_\psi(x)$ is represented  in tabular form
for simplicity:
$$V_\psi=[\psi_1,\ldots,\psi_N]$$ where $\psi\in \Real^N$.

\subsection{Loopy Markov chain}
We  test a toy example of the three-state loopy Markov chain
which was  considered in \cite{blanchet2016analysis,zheng2014stochastic}.
The  transition probability matrix 
for the  four state $\{0,1,2,3\}$ is 
$$
\tilde{K}=
\left[\begin{array}{cccc}
1 & 0 & 0 &0\\
\epsilon & \frac{1-\epsilon}{3} & \frac{1-\epsilon}{3}&\frac{1-\epsilon}{3} \\
\epsilon & \frac{1-\epsilon}{3} & \frac{1-\epsilon}{3}&\frac{1-\epsilon}{3}\\
\epsilon & \frac{1-\epsilon}{3} & \frac{1-\epsilon}{3}&\frac{1-\epsilon}{3}
\end{array}\right],\quad\epsilon\in(0,1).
$$
The state 0 is the absorbing state $\partial$ and $\mathcal{E}=\{1,2,3\}$.
$K$ is the sub-matrix of $\tilde{K}$ corresponding to the states
$\{1,2,3\}$.
With  the probability $\epsilon$,
the process exits $\mathcal{E}$ directly from state 1, 2 or 3. The true quasi-stationary distribution of this example is the uniform distribution for any $\epsilon$.

In order to show the advantage of our algorithm, we consider two cases: (1) $\epsilon=0.1$ and  (2) $\epsilon=0.9$.
For a larger  $\epsilon$,
the original Markov  chain is very easy to exit so each iteration takes less time,
but  the 
convergence  rate 
of Vanilla algorithm is slower.

To quantify the accuracy of the learned quasi-stationary distribution, we compute the $L_2$ norm of the error between the learned quasi-stationary distribution and the true values.
\begin{figure}[ht]
\centering
\setcounter {subfigure} {0} (a){\includegraphics[scale=0.35]{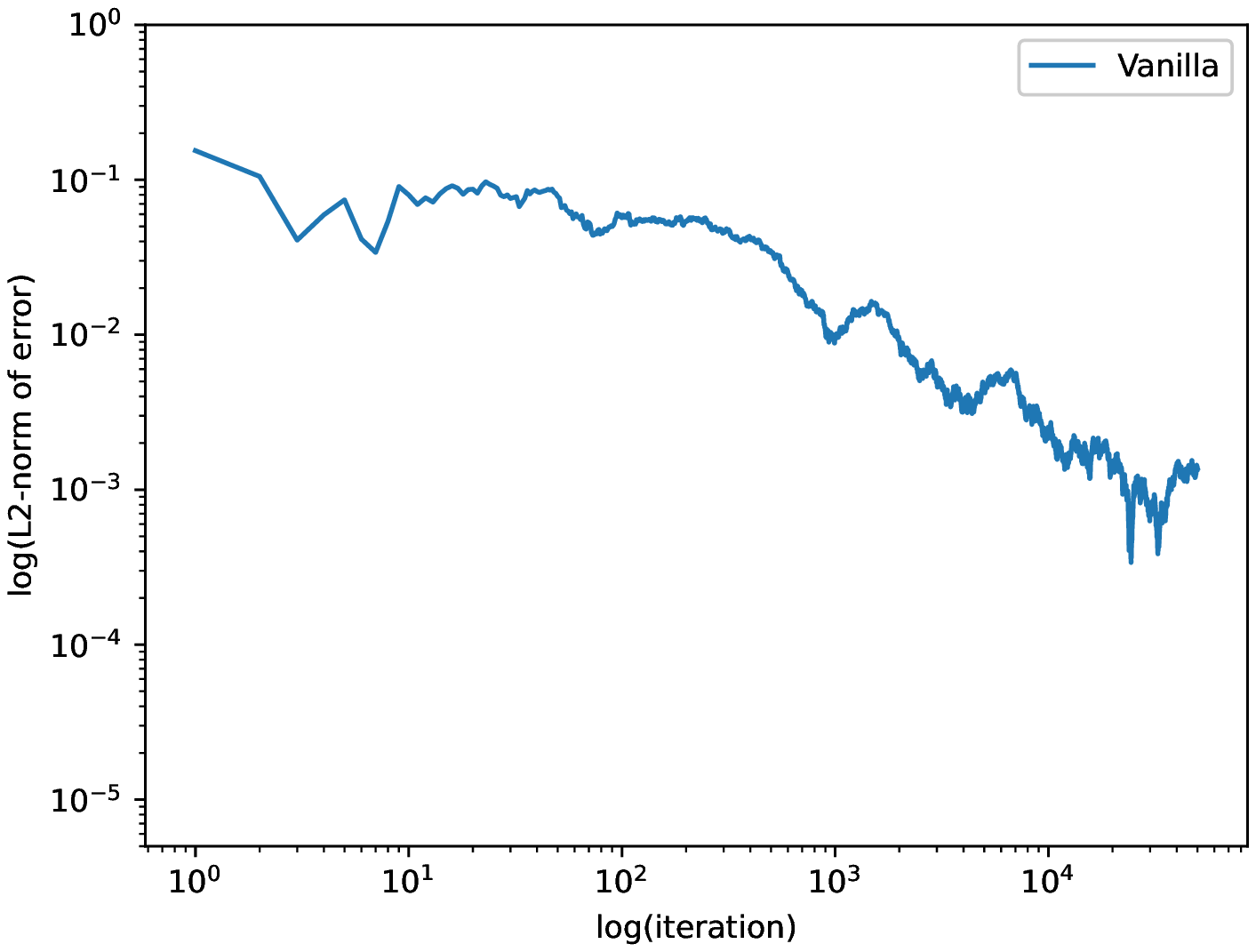}}
\setcounter {subfigure} {0} (b){\includegraphics[scale=0.35]{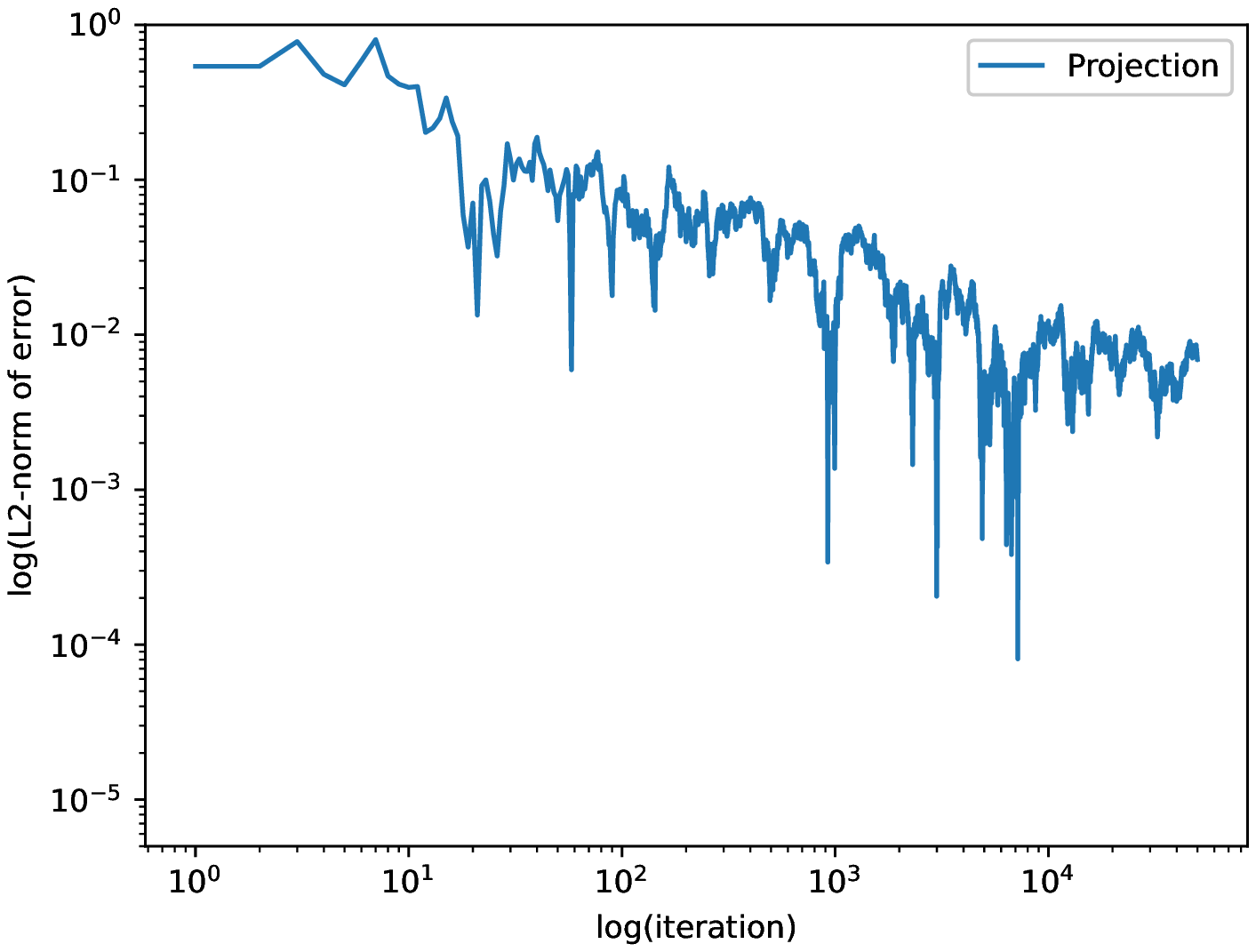}}
\\
\setcounter {subfigure} {0} (c){\includegraphics[scale=0.35]{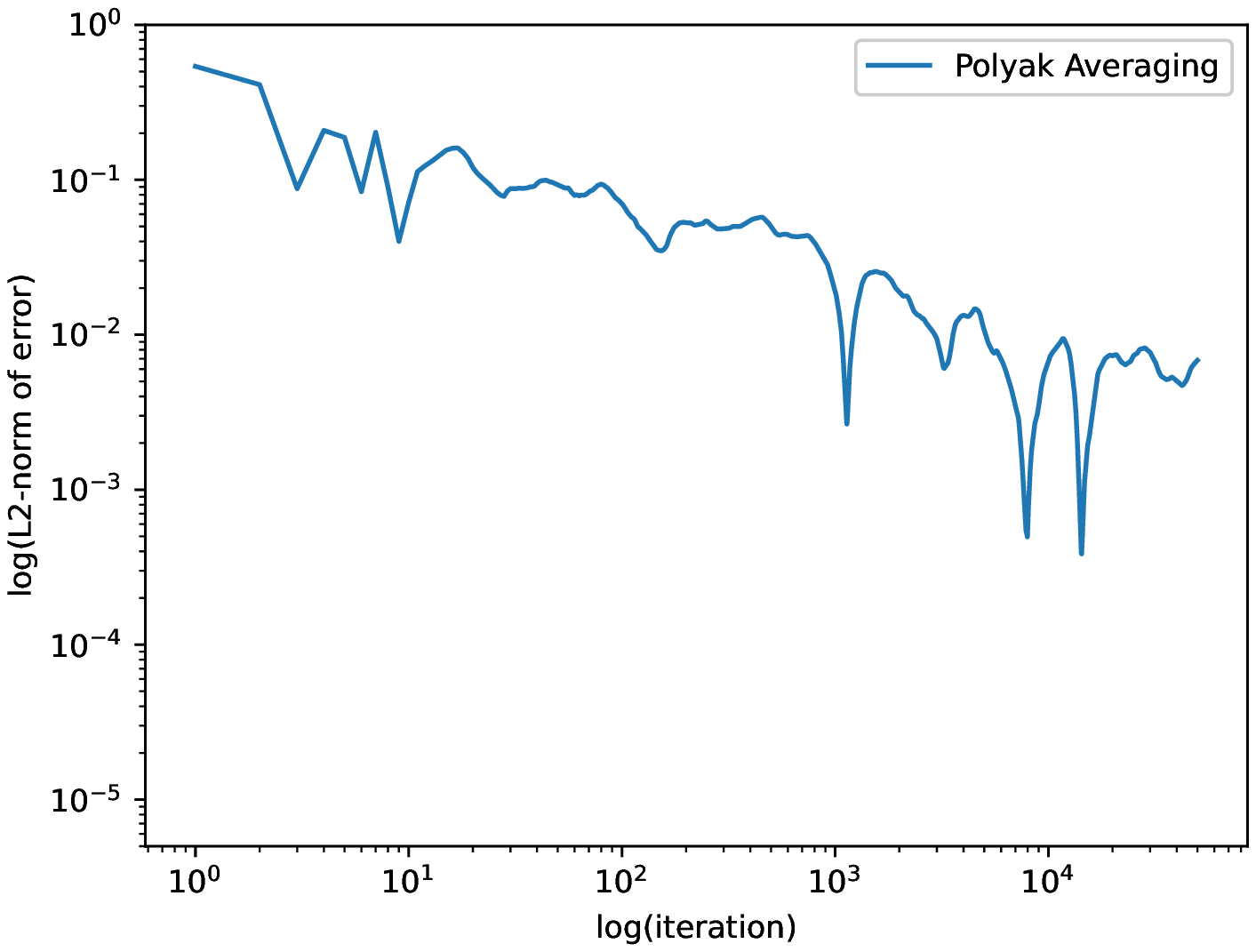}}
\setcounter {subfigure} {0} (d){\includegraphics[scale=0.35]{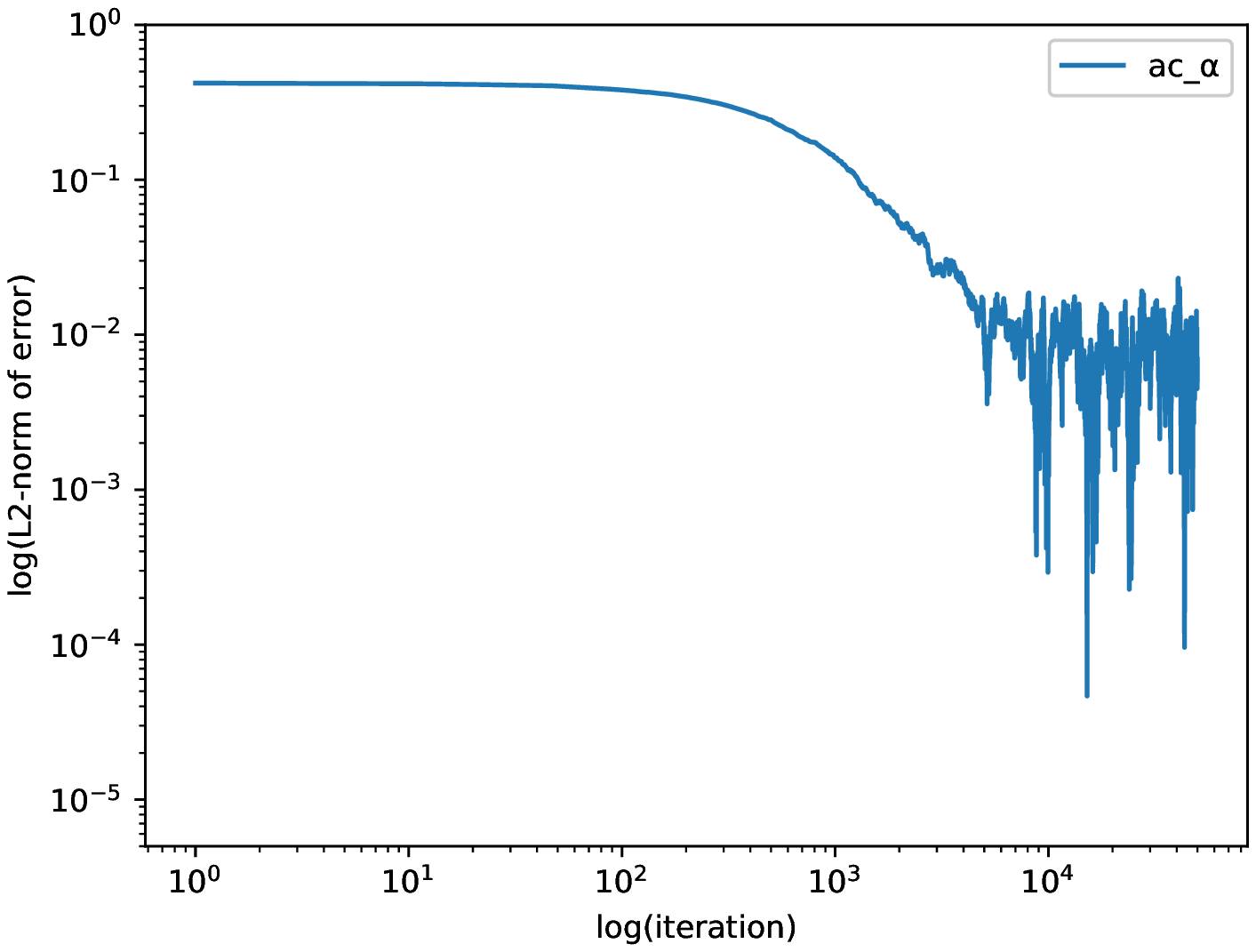}}
\caption{The loopy Markov chain example with $\epsilon= 0.1$. The figure shows the log-log plots of $L_2$-norm error
of Vanilla Algorithm (a), Projection Algorithm(b), Polyak Averaging Algorithm (c)
and our actor-critic algorithm (d). 
The iteration  for actor-critic algorithm
is defined as one step of gradient descent (``$t$'' in 
Algorithm \ref{dac1}).}
\label{fig:e1_1_1}
\end{figure}

In Figure \ref{fig:e1_1_1}, we compute the QSD when $\epsilon =0.1$. We set the initial value $\theta_0=[-1,1],\psi_0=[0,0,0],r_0=0$, the learning rate $\eta_n^\theta=\max\{1/{n^{0.1}},0.2\},\eta_n^\psi=0.0001,\eta_n^r=0.0001$ and the batch size is 4. The step size for Projection Algorithm is $\epsilon_n=n^{-0.99}$.  Figure \ref{fig:e1_2_1} is for the case when $\epsilon=0.9$   We set the initial value $\theta_0=[4,-2],\psi_0=[0,0,0],r_0=0$, the learning rate $\eta_n^\theta=0.04,\eta_n^\psi=0.0001,\eta_n^r=0.0001$ and the batch size is 32. The step size for Projection Algorithm is $\epsilon_n=n^{-0.99}$.

\begin{figure}[ht]
\centering
\setcounter {subfigure} {0} (a){\includegraphics[scale=0.35]{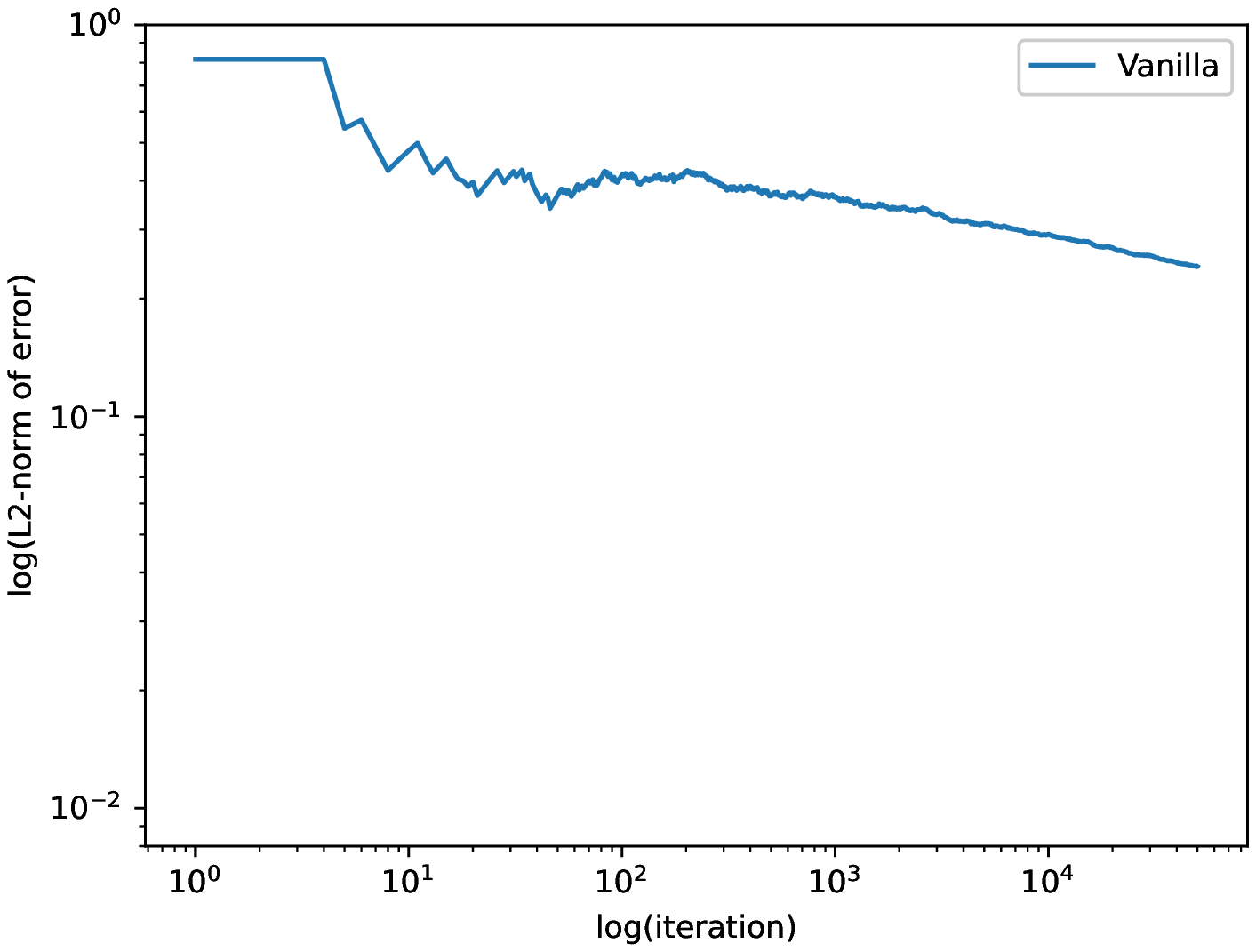}}
\setcounter {subfigure} {0} (b){\includegraphics[scale=0.35]{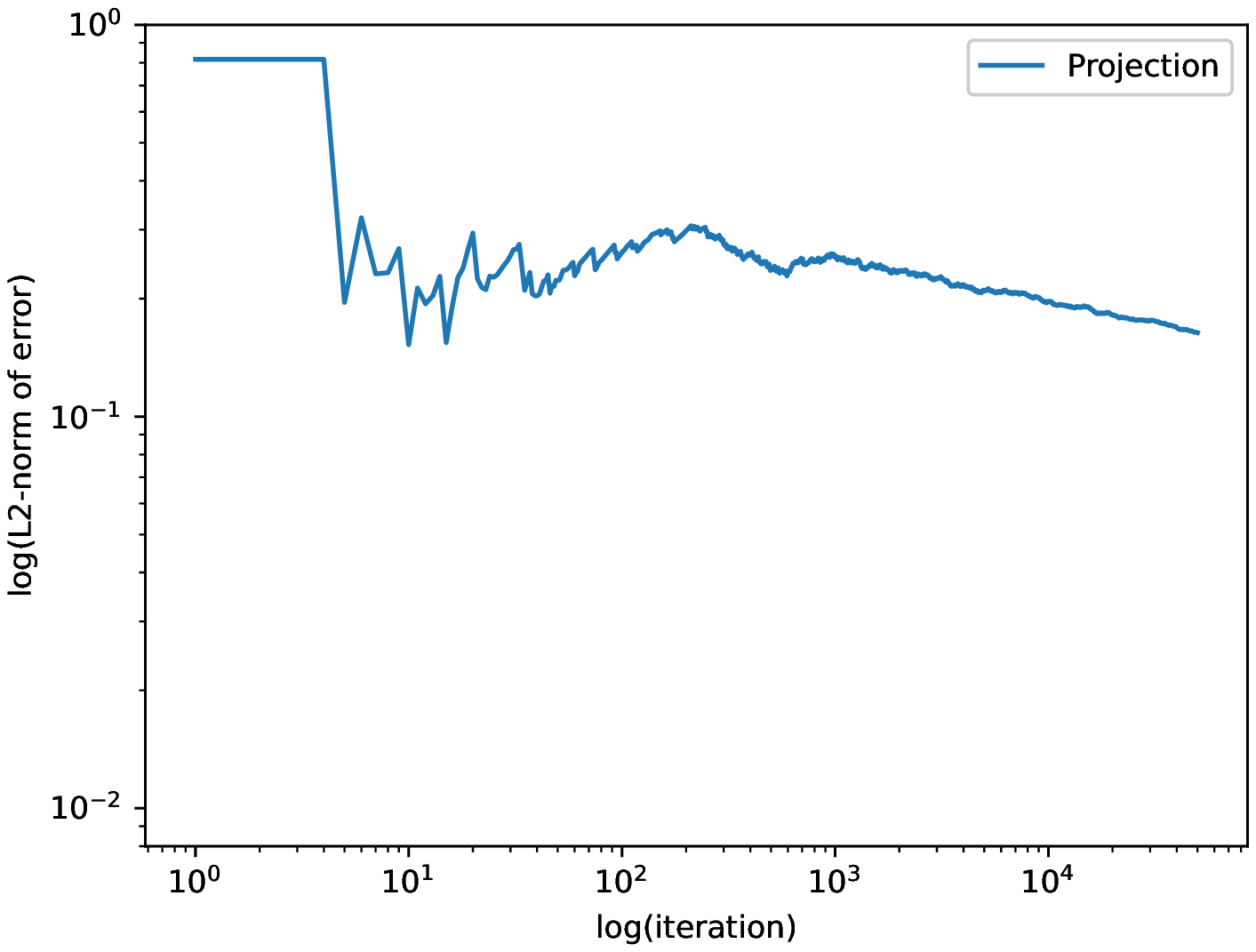}}
\\
\setcounter {subfigure} {0} (c){\includegraphics[scale=0.35]{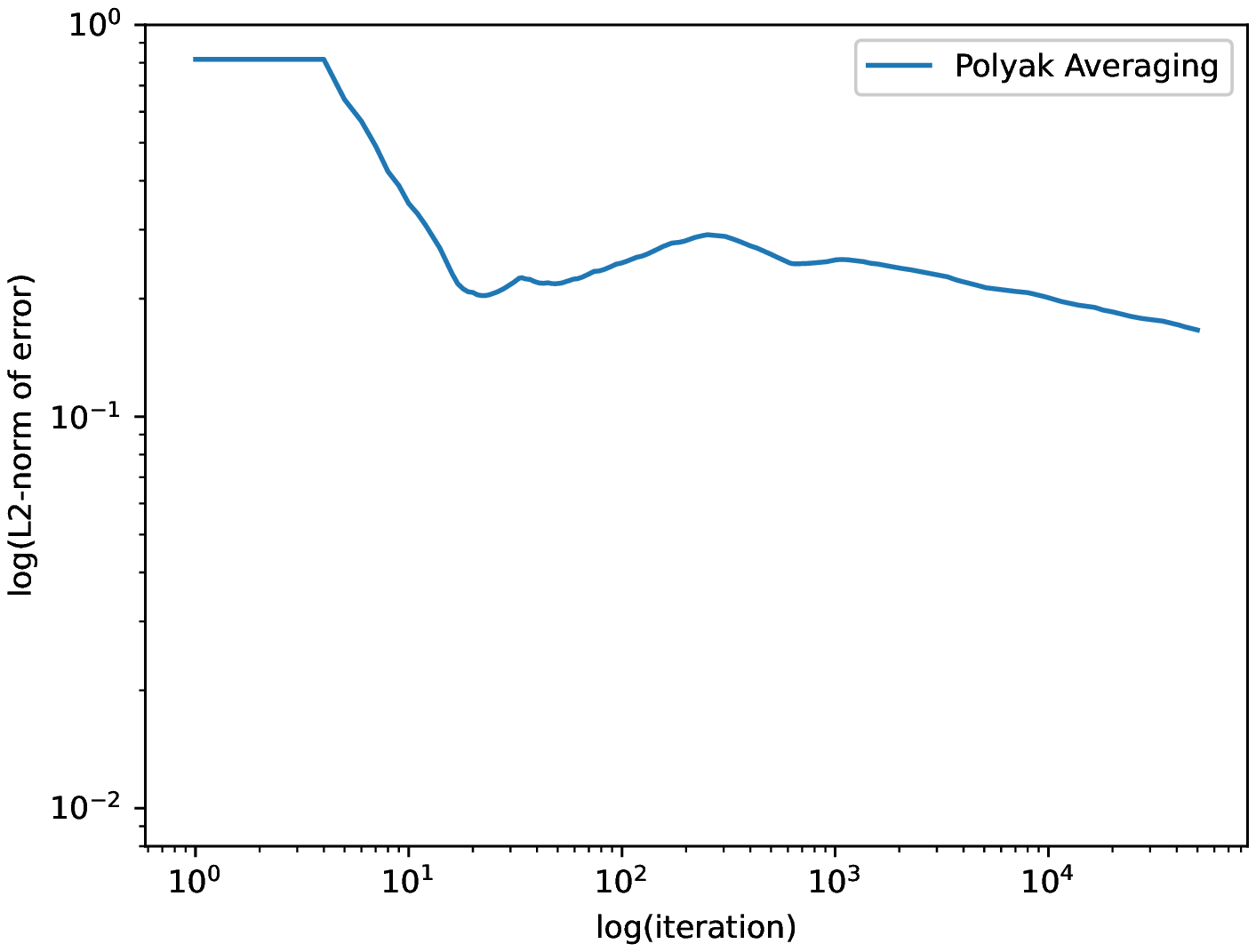}}
\setcounter {subfigure} {0} (d){\includegraphics[scale=0.35]{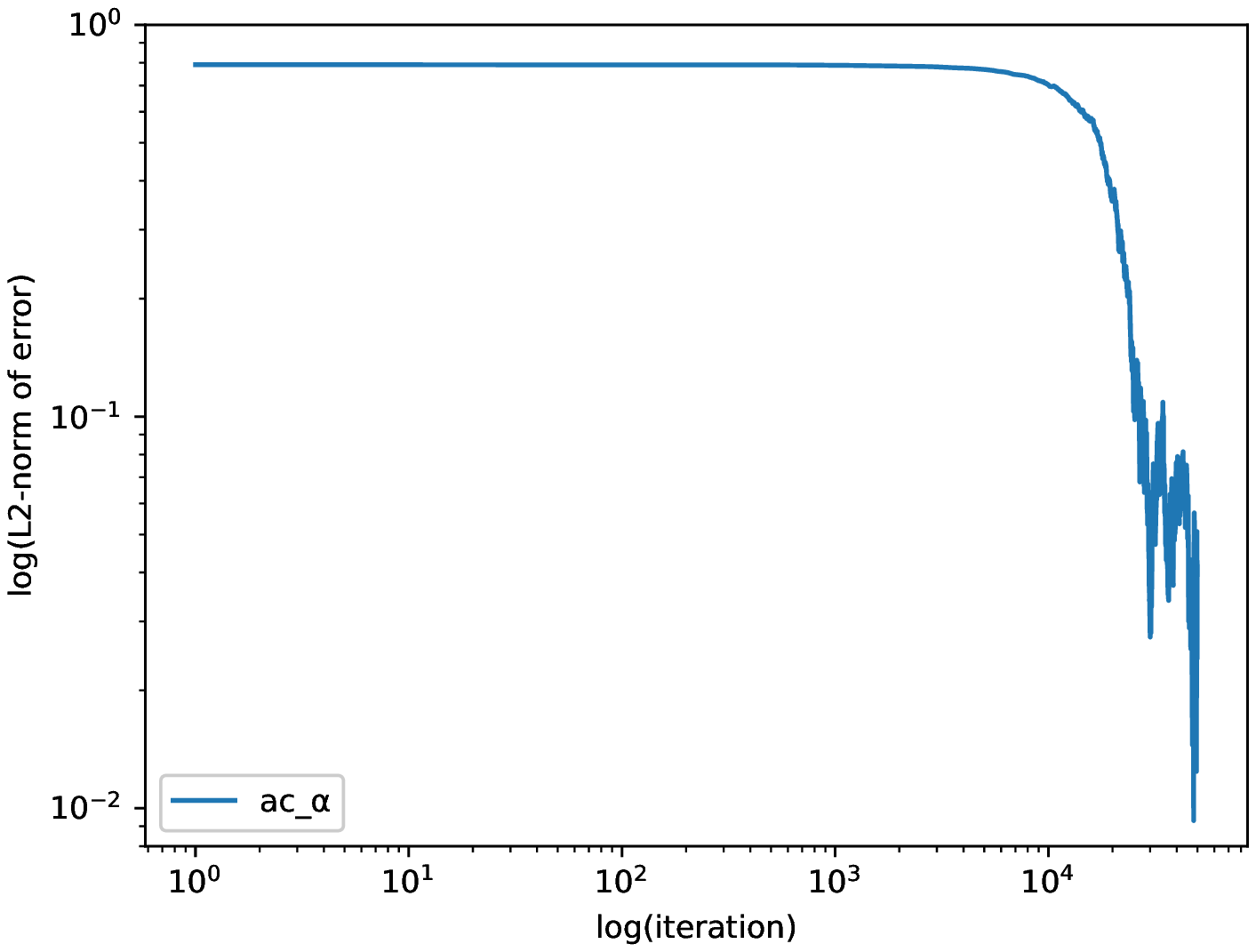}}
\caption{The loopy Markov chain example with $\epsilon= 0.9$. The figure shows the log-log plots of $L_2$-norm error
of Vanilla Algorithm (a), Projection Algorithm(b), Polyak Averaging Algorithm (c)
and our actor-critic algorithm (d). }
\label{fig:e1_2_1}
\end{figure}

\subsection{M/M/1/N queue with finite capacity and absorption}
Our second example is    a M/M/1 queue with finite queue capacity. The 0 state has been set as an absorbing state. The transition probability matrix on $\{0,\ldots,N\}$ takes  the  form
$$ \tilde{K}=
\left[\begin{array}{cccccccc}
1&0&0&0&0&\ldots&0&0\\
\mu_1&0 & \lambda_1 & 0 & 0 & \ldots & 0 & 0 \\
0&\mu_2 & 0 & \lambda_2 & 0 & \ldots & 0 & 0 \\
0&0 & \mu_3 & 0 & \lambda_3 & \ldots & 0 & 0 \\
\vdots&\vdots & \vdots & \vdots & \vdots & & \vdots & \vdots \\
0&0 & 0 & 0 & 0 & & 0 & \lambda_{N-1} \\
0&0 & 0 & 0 & 0 & \ldots & 1 & 0
\end{array}\right]
$$
where $\lambda_i = \frac{\rho_i}{\rho_i+1}$, $\mu_i = \frac{1}{\rho_i+1}$, $i\in\{1,2,\cdots,N-1\}$.
$\rho_i>1$ means 
a higher  chance  to jump to right than to left. A larger $\rho_i$ 
will have less probability of exiting $\mathcal{E}$.
Note that 
$K(x,\mathcal{E})=1$ for $x\in \{2,\ldots,N\}$.
So $K_{\alpha}(x,y)=K(x,y)$ for 
any $\alpha$ if $x\neq 1$
and $K_{\alpha}(1,y)=K(1,y)+\mu_1 \alpha(y)=\begin{cases} \lambda_1 + \mu_1 \alpha(1) & y=1,\\
\mu_1\alpha(y) & 2\leq y\leq N \\ \end{cases}.$
Then  ${R}_\theta(x,y)=-\ln{\frac{K_{\alpha_\theta}(x,y)}{K_{\beta_\theta}(x,y)}}=0$ if $x\neq 1$ and  by \eqref{eqn:gdac1}, the gradient is simplified as $$
 \nabla_\theta r(\theta)   =   \e_Y[\big(R_\theta(1,Y)-r(\theta)+V(Y)-V(1)\big) \nabla_\theta \ln K_{\alpha_\theta}(1,Y)
         +\nabla_\theta  \ln{K_{\beta_\theta}(1,Y)}]$$
where $Y$ follows the distribution  
$K_{\alpha}(1,\cdot)$.

\begin{figure}[ht]
\centering
\includegraphics[scale=0.35]{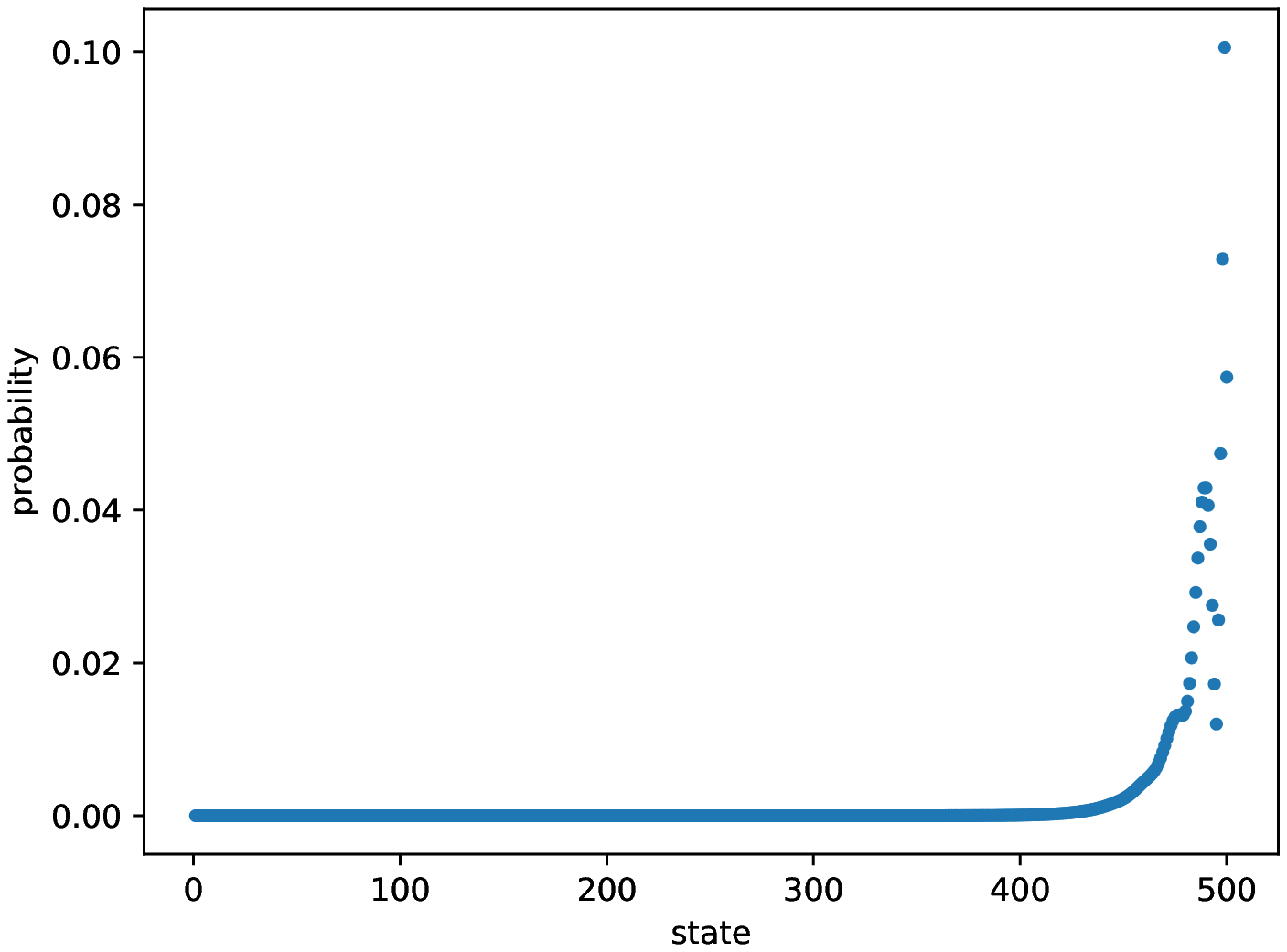} 
\includegraphics[scale=0.35]{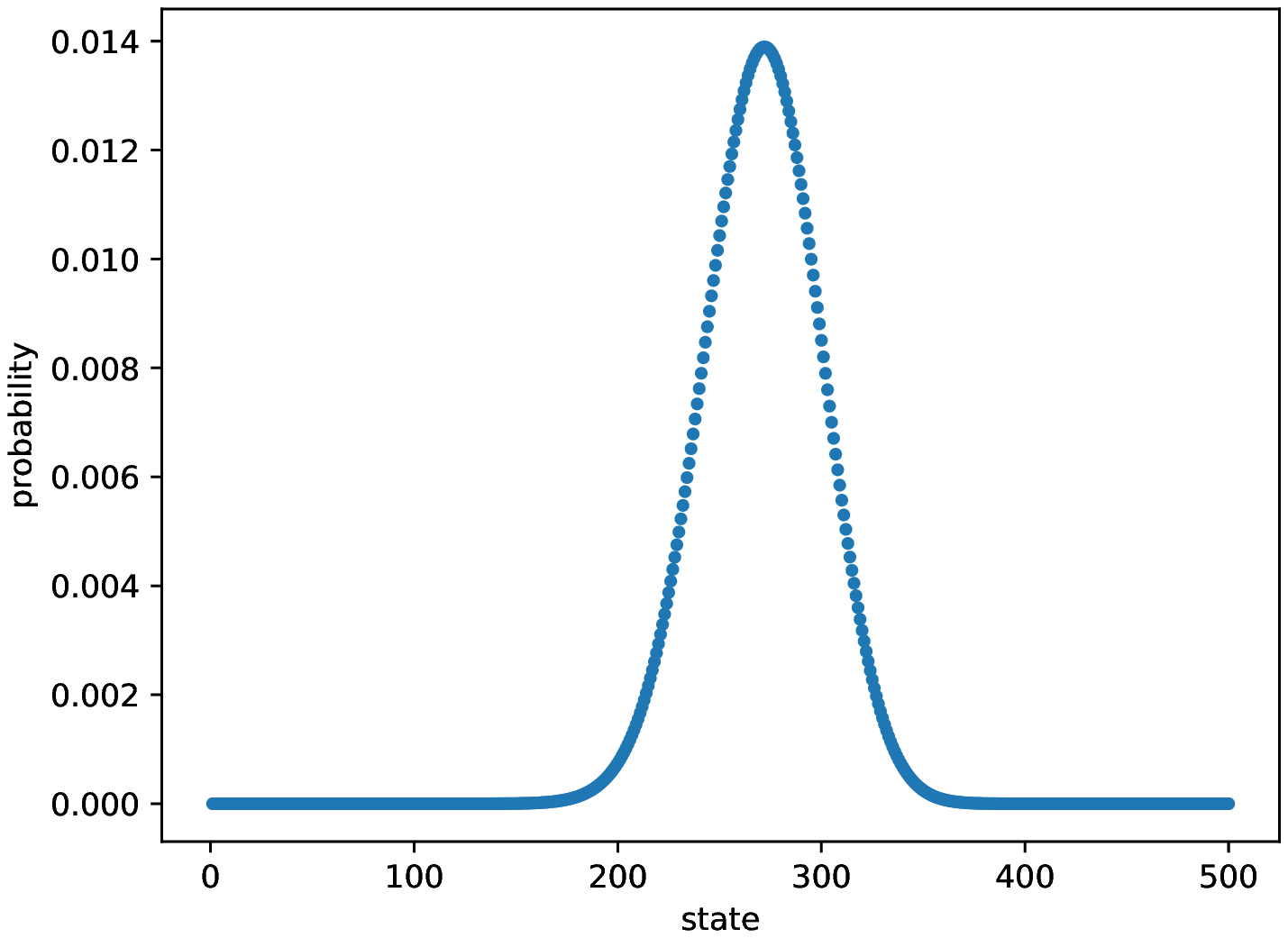} 

\caption{The QSD for M/M/1/500 queue with $\rho_i\equiv 1.25$
(left)
and $\rho_i=2-\frac{3}{2N-4}(i-1)$ (right).}
\label{fig:e2_1_3}
\end{figure}

We consider  two cases:
(1) a constant 
$\rho_i=1.25$ and (2)
a state-dependent  
$\rho_i=2-\frac{3}{2N-4}(i-1)$.  Note  $\rho_i=1$ gives  an equal  probability  of jumping to  left and  to right. So
in case (1),  there is a boundary  layer at the most right  end
and in case (2), we expect to see a peak of the QSD near 
$i\approx 2N/3$.
Figure \ref{fig:e2_1_3} shows the true QSD in both cases.
 We set $N=500$.

In    Figure \ref{fig:e2_1_1}, we consider the case when $\rho_i=1.25$ and compute the $L_2$  errors. We set the initial value $\theta_0^i=-35+\frac{35}{498}(i-1)$ for $i\in\{1,2,\dots,498\}$ and $\theta^{499}_0=3$, $\psi_0=[0,0,\dots,0]$, $r_0=0$, and the learning rate $\eta_n^\theta = 0.0003,\eta_n^\psi=0.0001,\eta_n^r=0.0001$ and the batch size is 64. The step size for Projection Algorithm is $\epsilon_n=n^{-0.95}$.
  Figure \ref{fig:e2_2_1} plots the
errors for the state-dependent
$\rho_i=2-\frac{3}{2N-4}(i-1)$.
We set the initial value $\theta_0^i=8+\frac{35}{250}(i-1)$ for $i\in\{1,2,\dots,250\}$, $\theta_0^{251}=44$, $\theta_0^i=43$ for $i\in\{252,\dots,305\}$, $\theta_0^{306}=48$, $\theta_0^{307}=42$ and $\theta_0^i=43-\frac{38}{293}(i-1)$ for $i\in\{308,309,\dots,499\}$, $\psi_0=[0,0,\dots,0],r_0=0$ and the learning rate is $\eta_n^\theta = 0.0002$, $\eta_n^\psi=0.0001,\eta_n^R=0.0001$ with batch size as 128. The step size for Projection Algorithm is $\epsilon_n=n^{-0.95}$.
 Both figures demonstrate the actor-critic algorithm performs quite well on this example.

In Table \ref{tb:e2_2}, we compared the CPU time of each algorithm in the M/M/1/500 queue when they obtain the accuracy at $2\times10^{-1}$. We found that our algorithm cost less time on this example.

\begin{figure}[ht]
\centering
\setcounter {subfigure} {0} (a){\includegraphics[scale=0.35]{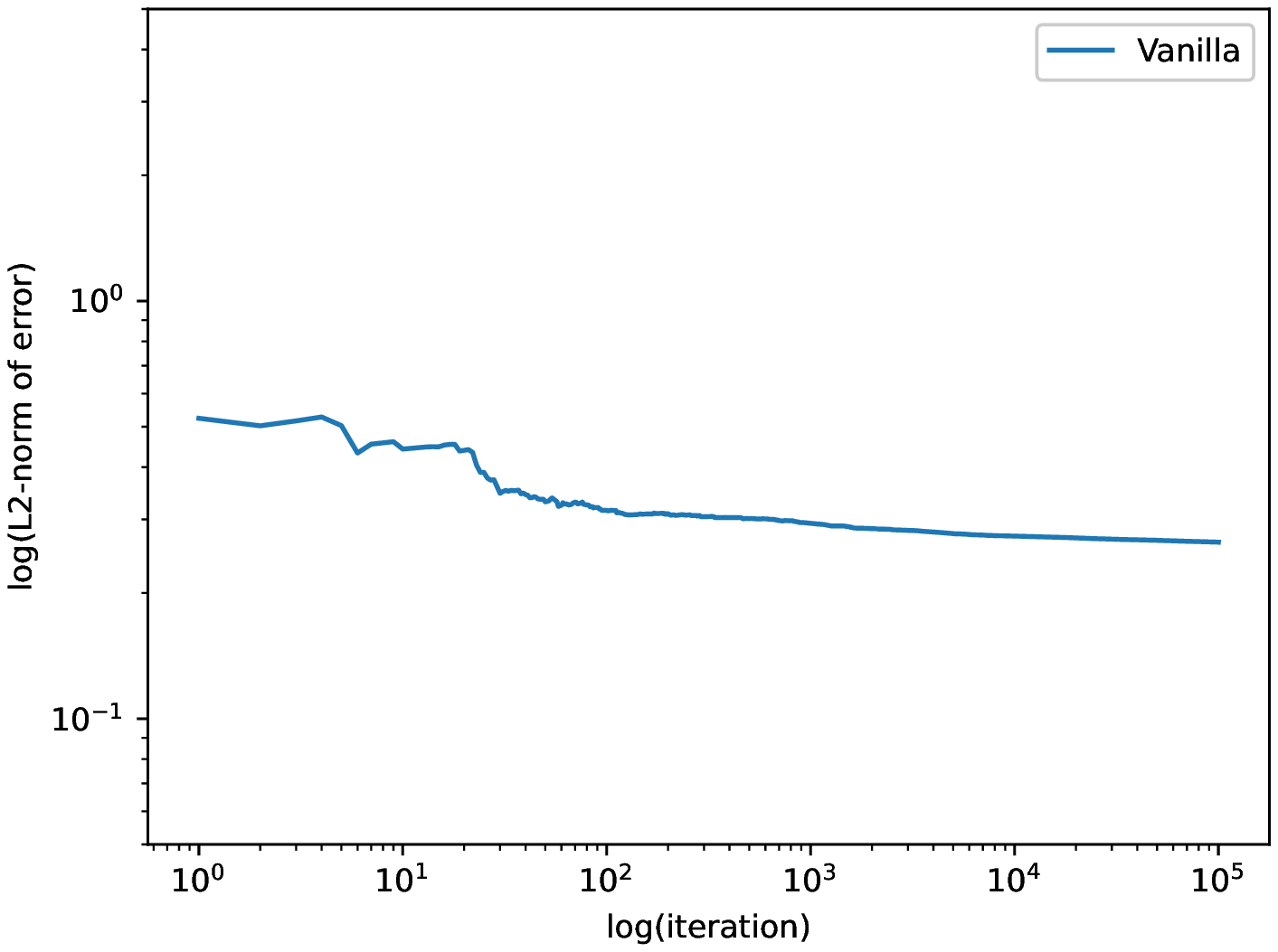}}
\setcounter {subfigure} {0} (b){\includegraphics[scale=0.35]{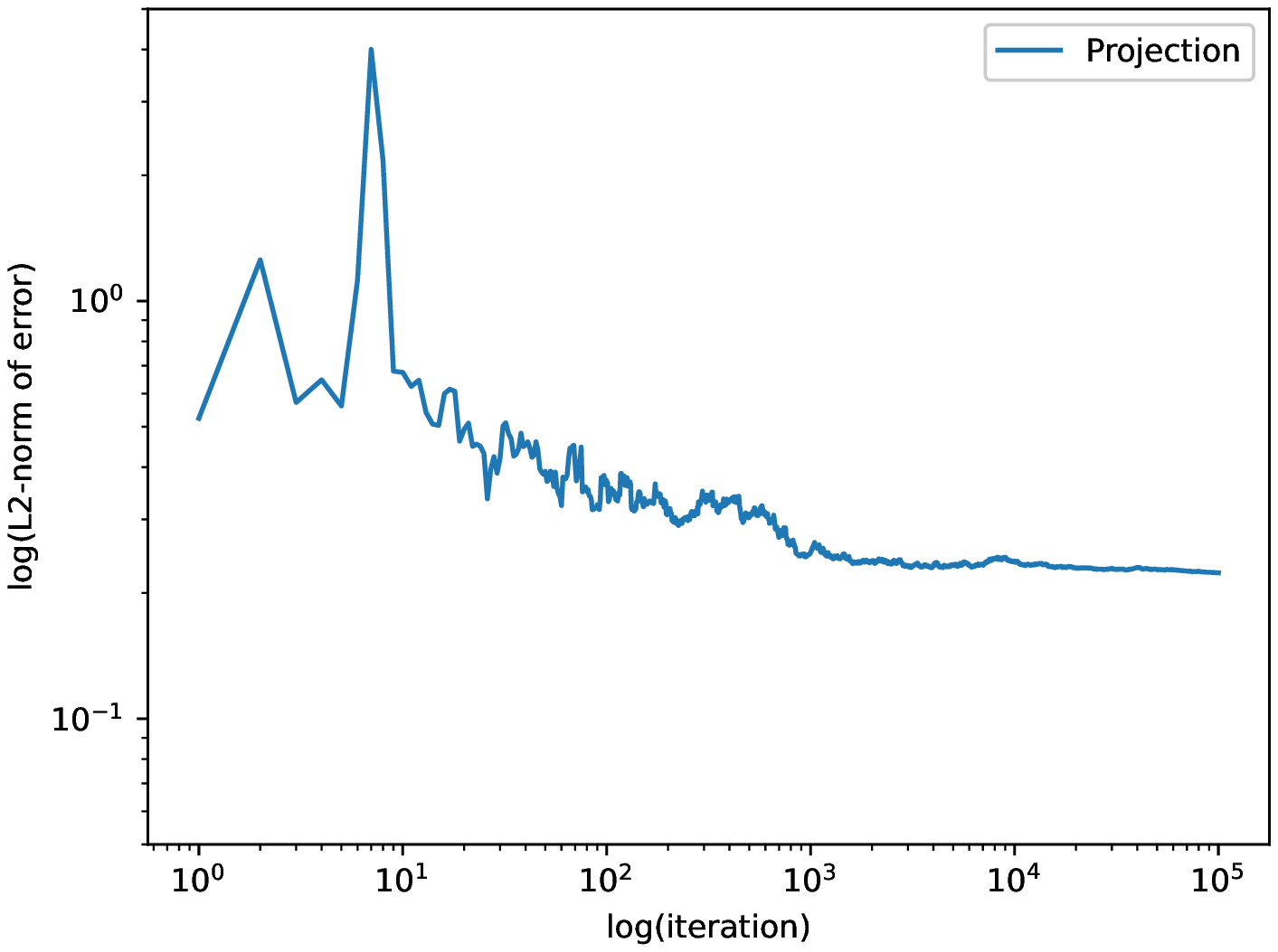}}
\\
\setcounter {subfigure} {0} (c){\includegraphics[scale=0.35]{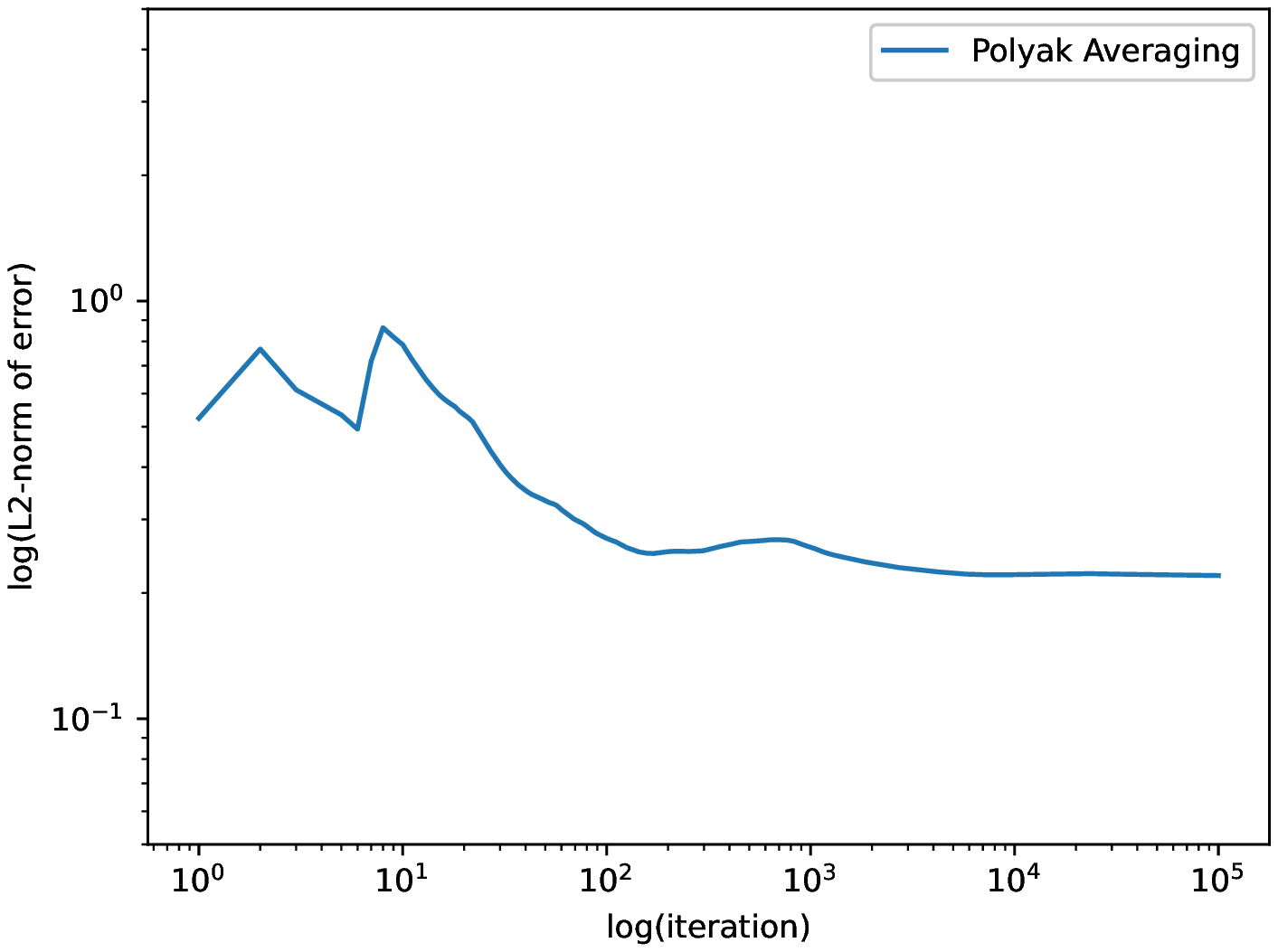}}
\setcounter {subfigure} {0} (d){\includegraphics[scale=0.35]{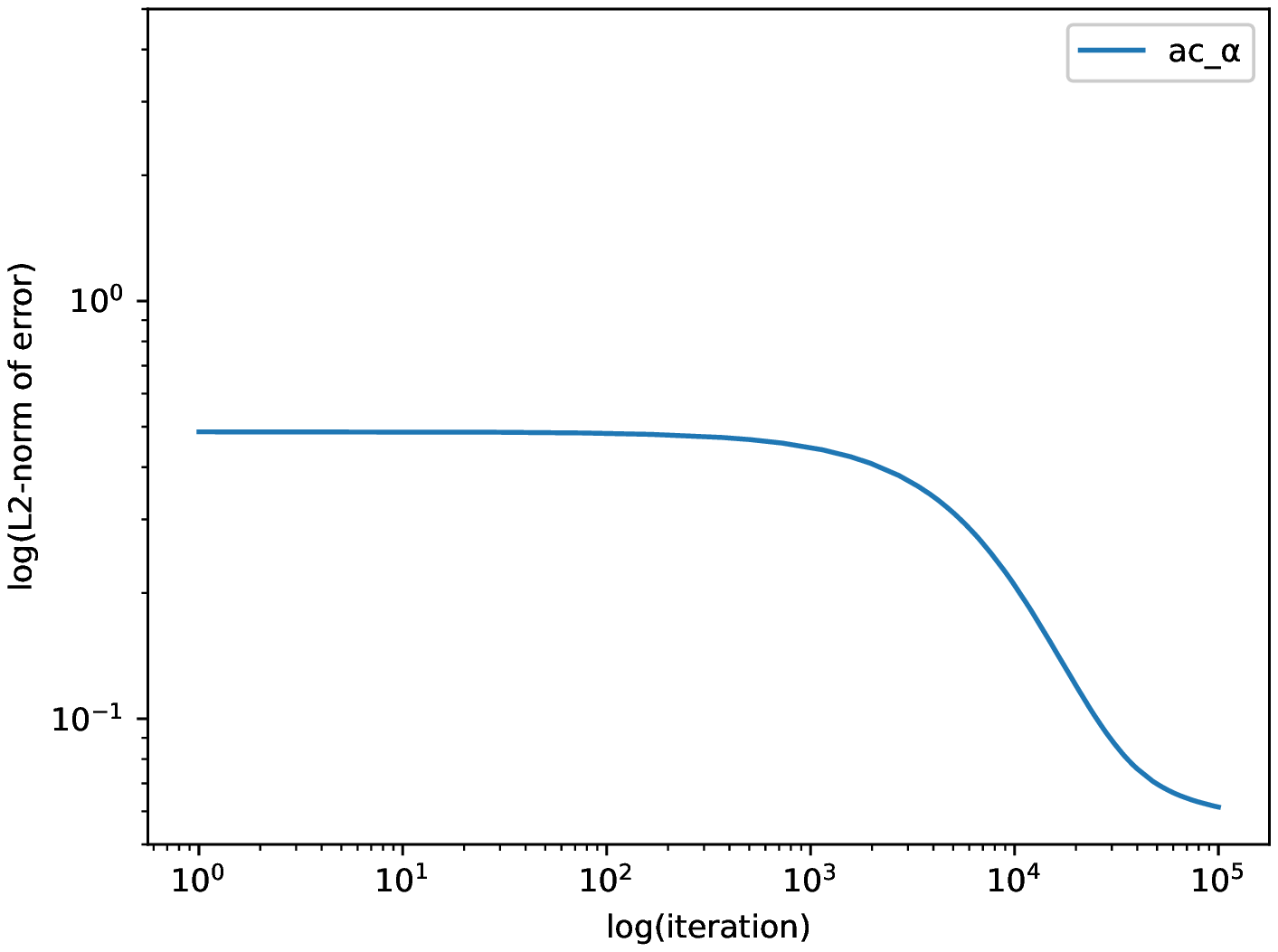}}
\caption{The M/M/1/500 queue with $\rho_i=1.25$. The figure shows the log-log plots of $L_2$-norm error
of Vanilla Algorithm (a), Projection Algorithm(b), Polyak Averaging Algorithm (c)
and our actor-critic algorithm (d).}  
\label{fig:e2_1_1}
\end{figure}

\begin{figure}[ht]
\centering
\setcounter {subfigure} {0} (a){\includegraphics[scale=0.35]{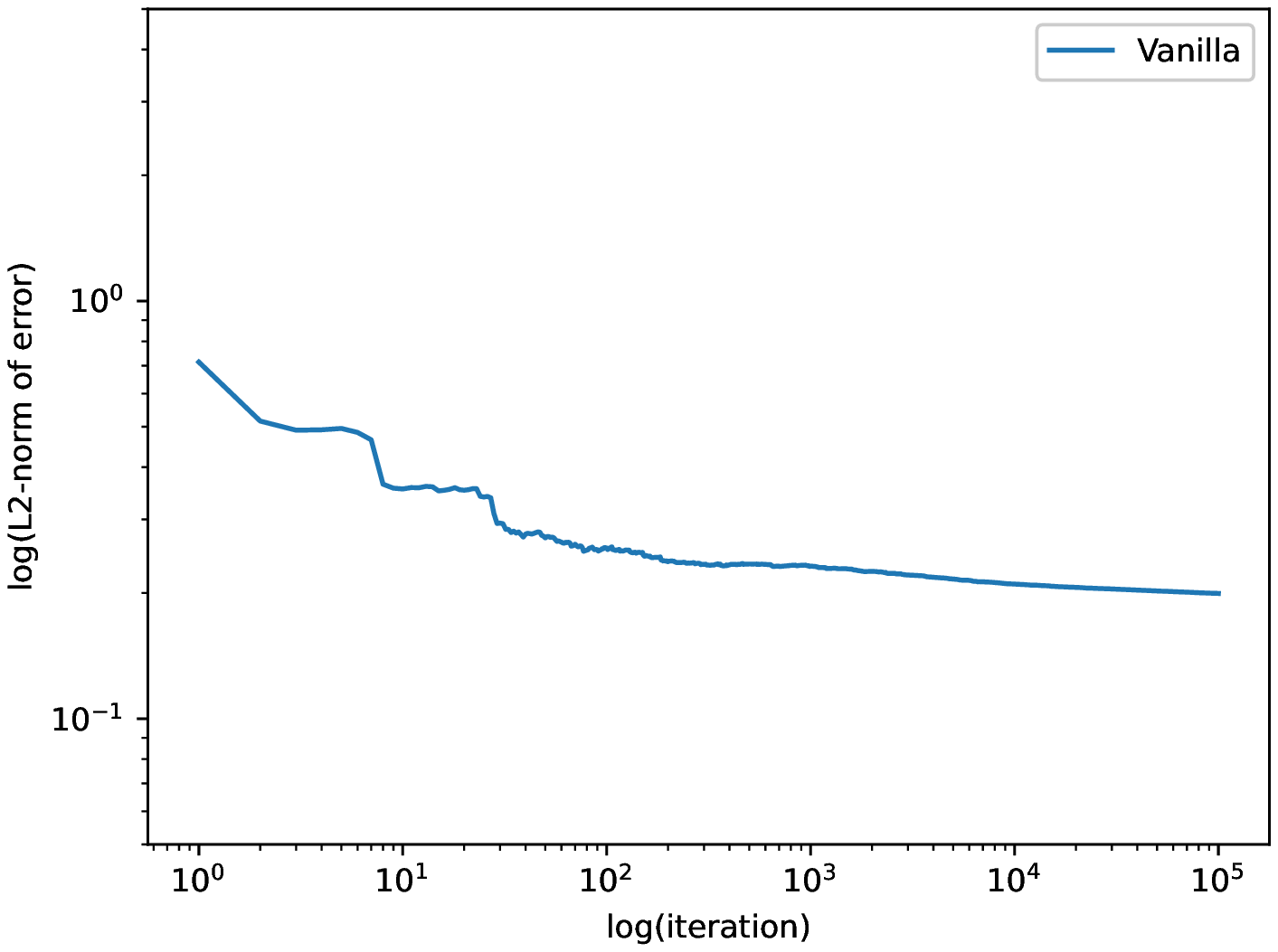}}
\setcounter {subfigure} {0} (b){\includegraphics[scale=0.35]{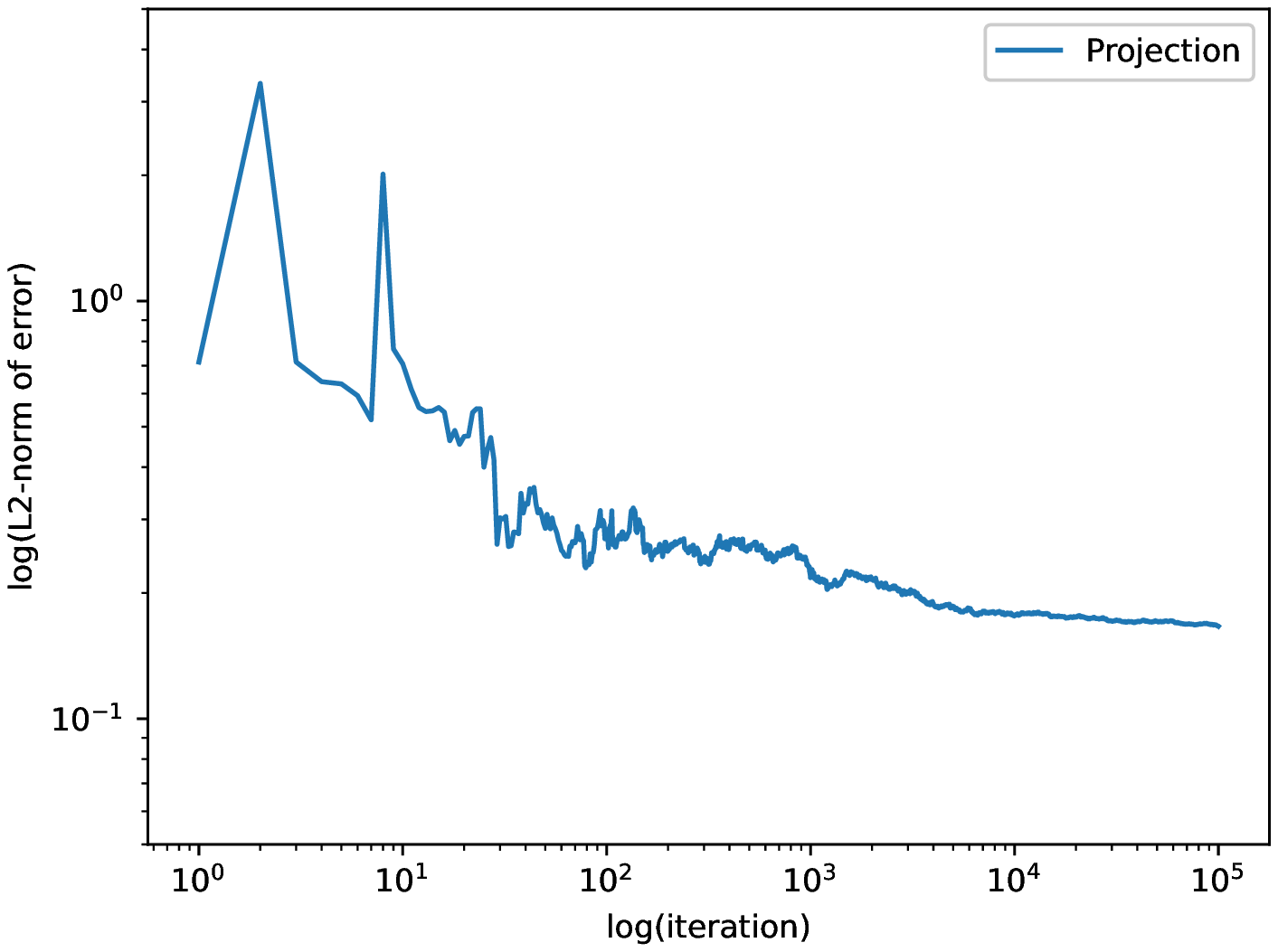}}
\\
\setcounter {subfigure} {0} (c){\includegraphics[scale=0.35]{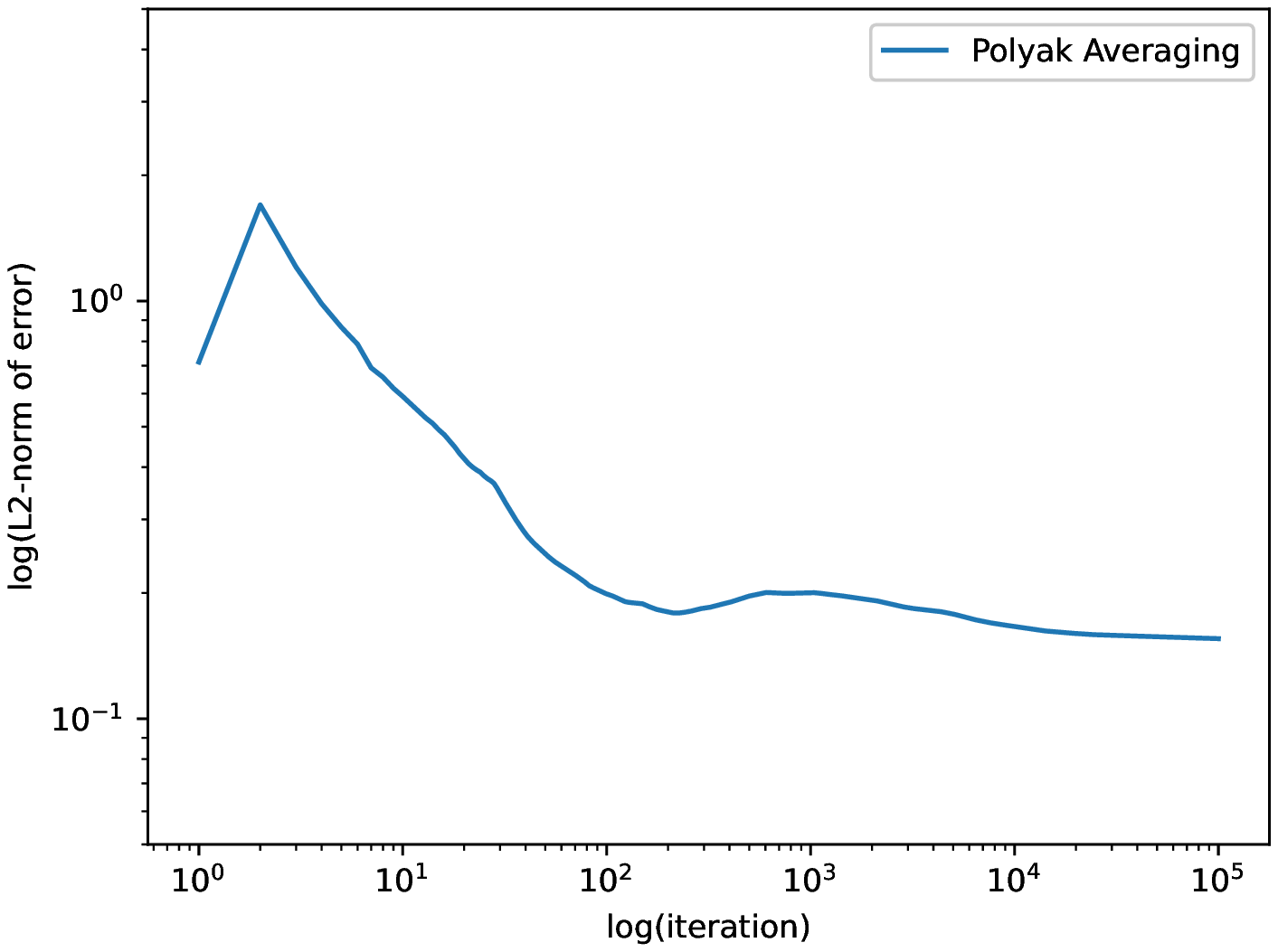}}
\setcounter {subfigure} {0} (d){\includegraphics[scale=0.35]{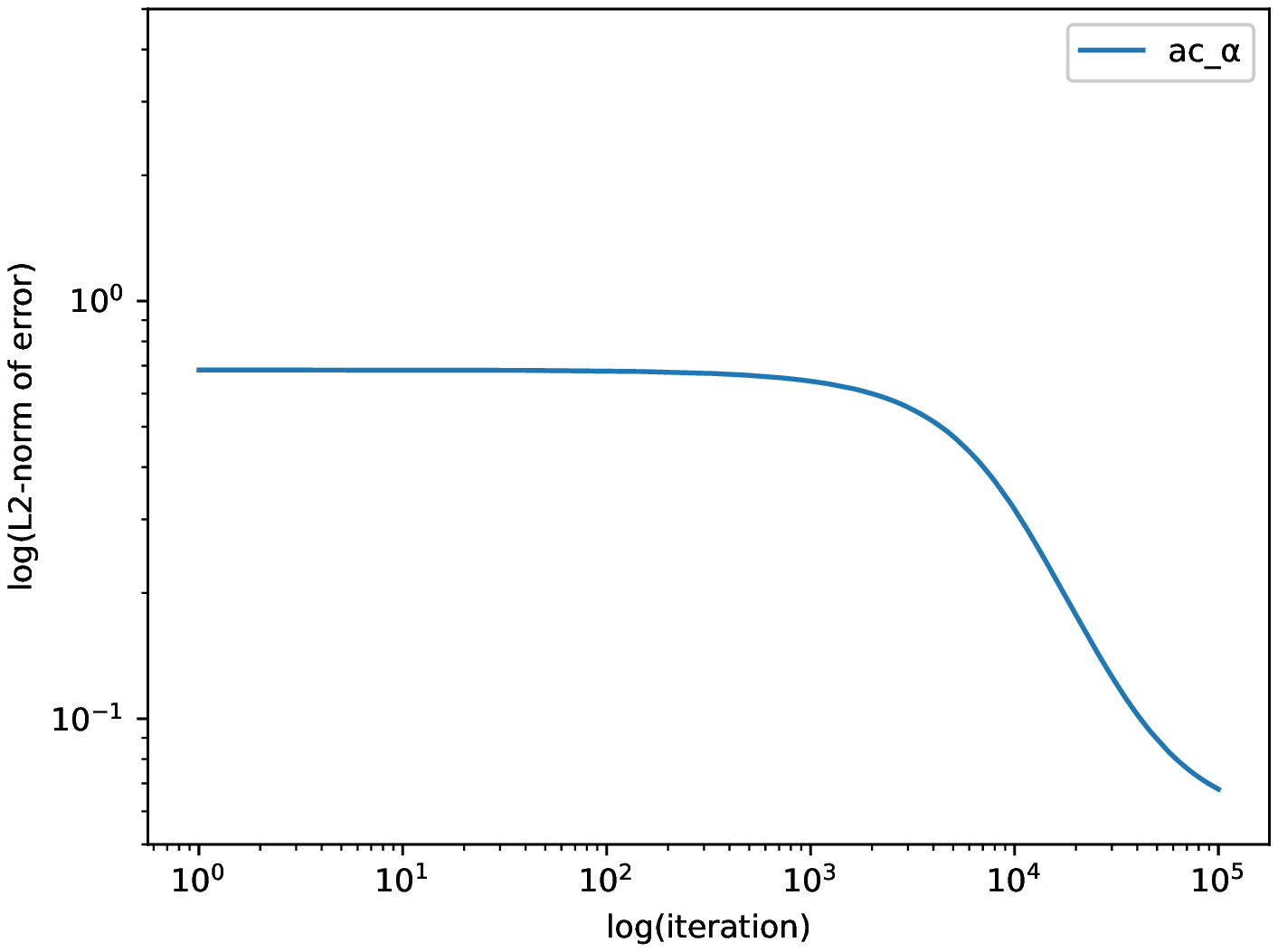}}
\caption{The M/M/1/500 queue with $\rho_i=2-\frac{3}{2N-4}(i-1)$. The figure shows the log-log plots of $L_2$-norm error
of Vanilla Algorithm (a), Projection Algorithm(b), Polyak Averaging Algorithm (c)
and our actor-critic algorithm (d). }
\label{fig:e2_2_1}
\end{figure}

\begin{table}[h!]
\begin{center}
\caption{The CPU time of each algorithm in the M/M/1/500 queue when they obtain the accuracy at $2\times10^{-1}$.}
\begin{tabular}{|c|c|c|c|c|}
\hline
Algorithm &Vanilla&Projection&Polyak Averaging&ac\_$\alpha$\\
\hline
Time(s)&1038.3279&429.6304  &505.2299 &186.9280 \\
\hline
Time(s)&753.9503&259.0671&268.5476&251.5370\\
\hline

\end{tabular}
\label{tb:e2_2}
\end{center}
\end{table}

\section{Summary and Conclusion}
In this paper, we propose a reinforcement learning (RL) method for the quasi-stationary distribution (QSD) in discrete time finite-state Markov chains. By minimizing the KL-divergence of two Markovian path distributions induced by the candidate distribution and the true target distribution, we introduce the formulation in terms of RL and derive the corresponding policy gradient theorem. We devise  an  actor-critic algorithm to learn the 
QSD in its parameterized form $\alpha_\theta$. This formulation of RL can get benefit from the development of the RL method and the optimization theory. We illustrated our actor-critic methods on two numerical examples by using simple tabular parametrization and gradient  descent optimization.  It  has been  observed that  the performance  of our method is more prominent  for large scale problem 

We only demonstrate the preliminary mechanism of the idea here, and there is much space left for improving the efficiency and extensions in future  works. The generalization  from  the  current consideration  of finite-state Markov chain    to the jump Markov process and the diffusion case is  in consideration.
More importantly, for very large or high dimensional state  space,   modern function approximation methods like
kernel methods or neural networks should be used for the distribution $\alpha_\theta$ and  the value function  $V_\psi$. The  recent  tremendous  advancement  of optimization techniques for  policy gradient in reinforcement  learning could also  contribute much to the efficiency improvement of our current formulation.

\section*{Acknowledgement}
LL acknowledges the support of NSFC 11871486. XZ acknowledges the support of Hong Kong RGC GRF     11305318.

\section*{Appendix}
In the appendix, we discuss the computation of the gradient of 
$\nabla_\theta \ln K_{\alpha_\theta}$
and 
$\nabla_\theta \ln K_{\beta_\theta}$.
Note $\nabla_\theta \alpha_\theta $
is straightforward since we model $\alpha$ in its parametrization form $\theta$. 
By definition \eqref{eqn:K_alpha},
$$
\nabla_\theta \ln K_{\alpha_\theta}(X_t,X_{t+1})=\frac{1-K(X_t,\mathcal{E})}{K_{\alpha_\theta}(X_t,X_{t+1})}
\nabla_\theta \alpha_\theta(X_{t+1})
$$
and
$$
\nabla_\theta \ln K_{\beta_\theta}(X_t,X_{t+1})=\frac{1-K(X_t,\mathcal{E})}{K_{\beta_\theta}(X_t,X_{t+1})}
\nabla \beta_\theta(X_{t+1}).
$$
where $K_\beta(x,y)=K(x,y)+(1-K(x,\mathcal{E}))\beta(y)$. The vector $K(x,\mathcal{E})$ for any $x$ can be  pre-computed and saved in tabular form. 
 
By  \eqref{eqn:beta}, 
the one-step distribution 
$\beta$ is computed below
$$\beta(X_{t+1})=
\sum_{x} \alpha(x) \bigg[ K(x,X_{t+1}) +
(1-K(x,\mathcal{E}))\alpha(X_{t+1}) \bigg]
\approx 
\frac{1}{n}
\sum_{i=1}^n 
 K(Z_i,X_{t+1}) +
(1-K(Z_i,\mathcal{E}))\alpha(X_{t+1}) $$
Here the samples $Z_i \sim \alpha$ and could be approximated by the stationary distribution $\mu$
and thus one may simply use the known sample
$X_t$ to replace $Z_i$ with $n=1$.

To find $\nabla_\theta \beta_\theta$, we use stochastic approximation again

 $$
\begin{aligned}
\nabla_\theta \beta_\theta(X_{t+1}) 
&=\sum_x \nabla \alpha_\theta(x) \left[K(x,X_{t+1})
+ (1-K(x,\mathcal{E})){\alpha_\theta}(X_{t+1})
\right]+ 
\left[\sum_{x} \alpha_\theta(x)   (1-K(x,\mathcal{E}))
\right] \nabla_\theta \alpha_\theta (y),
\\
&\approx
\frac{1}{n}
\sum_{i=1}^n 
\nabla \ln \alpha_\theta(Z_i) \left[K(Z_i,X_{t+1})
+ (1-K(Z_i,\mathcal{E})){\alpha_\theta}(X_{t+1})
\right]+ 
   (1-K(Z_i,\mathcal{E}))
 \nabla_\theta \alpha_\theta (X_{t+1}).
\end{aligned}
  $$

\bibliography{ref}
 
\bibliographystyle{siam} 

\end{document}